\documentclass{article}

\usepackage{arxiv}

\usepackage[utf8]{inputenc} 
\usepackage[T1]{fontenc}    
\usepackage{hyperref}       
\usepackage{url}            
\usepackage{booktabs}       
\usepackage{amsfonts}       
\usepackage{nicefrac}       
\usepackage{microtype}      
\usepackage{lipsum}		
\usepackage{graphicx}
\usepackage{natbib}
\usepackage{doi}

\usepackage{microtype}
\usepackage{graphicx}
\usepackage{subfigure}
\usepackage{wrapfig}
\usepackage{booktabs} 

\usepackage[ruled,vlined]{algorithm2e}

\usepackage{amsmath}
\usepackage{amssymb}
\usepackage{mathtools}
\usepackage{amsthm}
\usepackage{thmtools} 

\usepackage[capitalize,noabbrev]{cleveref}

\newtheorem{lemma}{Lemma}

\newtheorem{proposition}{Proposition}

\newtheorem{corollary}{Corollary}

\newtheorem{assumption}{Assumption}

\usepackage[textsize=tiny]{todonotes}

\newcommand{\out}[1]{}

\usepackage{amsmath,amsfonts,bm}

\def\eqref#1{equation~\ref{#1}}

\def\1{\bm{1}}

\DeclareMathAlphabet{\mathsfit}{\encodingdefault}{\sfdefault}{m}{sl}
\SetMathAlphabet{\mathsfit}{bold}{\encodingdefault}{\sfdefault}{bx}{n}

\usepackage{booktabs}
\usepackage{amssymb,amsmath} 
\usepackage{multirow}

\newcommand{\ours}{Diffusion-BBO}

\newcommand\dif{\mathop{}\!\mathrm{d}}

\DeclareMathOperator*{\argmax}{argmax}

\def \rT{\mathrm{T}}

\newcommand*\R[0]{\mathbb{R}}

\newcommand*{\E}{\ensuremath{{\mathbb{E}}}}
\newcommand\numberthis{\addtocounter{equation}{1}\tag{\theequation}}

\newcommand*\lrn[1]{\left\|#1\right\|}

\title{Diffusion-BBO:  Diffusion-Based Inverse Modeling for \\ Online Black-Box Optimization}

\author{ {\hspace{0.1mm}Dongxia Wu}\\
    \texttt{University of California San Diego}\\
    \texttt{La Jolla, CA}\\
	\texttt{dowu@ucsd.edu} \\
	\And{\hspace{0.1mm}Nikki Lijing Kuang}\\
    \texttt{University of California San Diego}\\
    \texttt{La Jolla, CA}\\
	\texttt{l1kuang@ucsd.edu} \\
    \And{\hspace{0.1mm}Ruijia Niu}\\
    \texttt{University of California San Diego}\\
    \texttt{La Jolla, CA}\\
	\texttt{rniu@ucsd.edu} \\
    \And{\hspace{0.1mm}Yi-An Ma}\\
    \texttt{University of California San Diego}\\
    \texttt{La Jolla, CA}\\
	\texttt{yianma@ucsd.edu} \\
    \And{\hspace{0.1mm}Rose Yu}\\
    \texttt{University of California San Diego}\\
    \texttt{La Jolla, CA}\\
    \texttt{roseyu@ucsd.edu}\\
}

\date{}

\newcommand{\blfootnote}[1]{\begingroup%
\renewcommand\thefootnote{}\footnotetext{#1}%
\addtocounter{footnote}{-1}%
\endgroup}

\renewcommand{\headeright}{}
\renewcommand{\undertitle}{}
\renewcommand{\shorttitle}{Diffusion-BBO:  Diffusion-Based Inverse Modeling for Online Black-Box Optimization}

\hypersetup{
pdftitle={Diffusion-BBO:  Diffusion-Based Inverse Modeling for Online Black-Box Optimization},
pdfauthor={Dongxia Wu, Nikki Lijing Kuang, Ruijia Niu, Yi-An Ma, Rose Yu},
pdfkeywords={Diffusion Models, Black-Box Optimization, AI for Science},
}

\begin{document}
\maketitle

\blfootnote{Preprint. Under Review.}

\begin{abstract}
\begin{enumerate}
Online black-box optimization (BBO) aims to optimize an objective function by iteratively querying a black-box oracle in a sample-efficient way.
While prior studies focus on forward approaches such as Gaussian Processes (GPs) to learn a surrogate model for the unknown objective function, they struggle with steering clear of out-of-distribution and invalid designs in scientific discovery tasks.
Recently, inverse modeling approaches that map the objective space to the design space with conditional diffusion models have demonstrated impressive capability in learning the data manifold.
However, these approaches proceed in an \textit{offline} fashion with pre-collected data. 
How to design inverse approaches for \textit{online} BBO to \textit{actively} query new data and improve the sample efficiency remains an open question.
In this work, we propose \ours{}, a sample-efficient online BBO framework leveraging the conditional
diffusion model as the \textit{inverse} surrogate model.
\ours{} employs a novel acquisition function \textit{Uncertainty-aware Exploration} (UaE) to propose scores in the objective space for conditional sampling. 
We theoretically prove that \ours{} with UaE achieves a near-optimal solution for online BBO.
We also empirically demonstrate that \ours{} with UaE outperforms existing online BBO baselines across 6 scientific discovery tasks.

\end{enumerate}
\end{abstract}

\section{Introduction}
Practical problems in scientific discovery often involve optimizing a black-box objective function that is expensive to evaluate, seen in fields such as superconducting material design \citep{fannjiang2020autofocused},
DNA sequence optimization \citep{daniel2015atgme}, and molecular design \citep{sanchez2018inverse}. How to achieve a near-optimal solution while minimizing function evaluations is thus a major task in black-box optimization (BBO). To improve sample efficiency, prior works in BBO have largely focused on the \textit{online} setting where the algorithm can iteratively select candidates in the design space and query the black-box function for evaluation \citep{turner2021bayesian, zhang2021unifying, hebbal2019bayesian, mockus1974bayesian}. Most existing algorithms belong to the class of forward approaches, including Bayesian optimization (BO) \citep{kushner1964new, mockus1974bayesian, wu2023disentangled, frazier2018tutorial}, bandit algorithms \citep{agrawal2012analysis, karbasi2023langevin}, conditional sampling approaches \citep{brookes2019conditioning, gruver2024protein, stanton2022accelerating}, in which a surrogate model is built to iteratively approximate and optimize the black-box function.

However, these approaches often fail in scenarios where valid designs, such as valid protein sequences or molecular structures \citep{kumar2020model}, are confined in a small subspace. 
Such optimization problems become challenging as the optimizer needs to eliminate out-of-distribution or invalid designs.
Recently, a novel set of methods, termed \textit{inverse approaches}, have been proposed to address this issue \citep{kumar2020model,krishnamoorthy2023diffusion,kim2024bootstrapped}. Specifically, an inverse surrogate model is constructed to learn a mapping from the objective space back to the design space. 
By leveraging the state-of-the-art generative models, such as diffusion models \citep{song2020score}, these approaches can effectively learn the data distribution in the design space and facilitate optimization within the data manifold \citep{kong2024diffusion}. 
Promising performance has been achieved in \textit{offline} BBO settings \citep{kumar2020model, lu2023degradation, wang2018offline}, assuming an access to a fixed pre-collected dataset. 

Despite the advancements of inverse approaches in offline settings, extending existing methods to \textit{online} BBO remains challenging due to two unresolved issues: (1) the lack of rigorous uncertainty quantification over inverse surrogate models, and (2) the absence of principled methods to leverage the model uncertainty for acquisition function design in the objective space. These limitations hinder the application of inverse approaches for online BBO, where enhancing sample efficiency is crucial as data acquisition is expensive in tasks such as scientific discovery. 
In this paper, we propose \ours{} to tackle the above issues. It is a sample-efficient online BBO framework that leverages the conditional diffusion model as the inverse surrogate model. 
\ours{} consists of a novel acquisition function \textit{Uncertainty-aware Exploration} (UaE), which 
strategically propose the objective scores within the training data distribution for conditional sampling in the design space.
We summarize our main contributions as follows:
\begin{itemize}
    \item We provide an uncertainty decomposition into epistemic uncertainty and aleatoric uncertainty for conditional diffusion models. We rigorously analyze how uncertainty propagates throughout the denoising process of the conditional diffusion model. 
    \item We design the acquisition function UaE for \ours{} based on the quantified uncertainty of conditional diffusion models. 
    Theoretically, we prove that 
    UaE achieves a balance between targeting higher objective values and minimizing epistemic uncertainty, leading to optimal optimization outcomes.
    \item We demonstrate that \ours{} with UaE achieves state-of-the-art online BBO performance across 6 scientific discovery tasks in Design-Bench and molecular discovery.
\end{itemize}

\section{Related Work}
\label{sec:related}
\paragraph{Online Black-box Optimization.}

For online Black-box Optimization (BBO), algorithms are designed to iteratively query the oracle function to obtain new training data \citep{turner2021bayesian, zhang2021unifying, hebbal2019bayesian, mockus1974bayesian}. In this setting, most existing algorithms belong to the class of forward approaches, including Bayesian optimization (BO) \citep{kushner1964new, mockus1974bayesian, wu2023disentangled, frazier2018tutorial}, bandit algorithms \citep{agrawal2012analysis, karbasi2023langevin}, and conditional sampling approaches \citep{brookes2019conditioning, gruver2024protein, stanton2022accelerating}. 
\citet{liu2024large} further integrate LLM capabilities into the BO framework to enable zero-shot warm-starting and enhance surrogate
modeling and candidate sampling.
Forward approaches build a surrogate model to approximate and optimize the black-box objective function. However, these approaches can struggle with capturing the data manifold in the design space and avoiding out-of-distribution and invalid inputs \citep{kumar2020model}. \citet{song2022general, zhang2021unifying} proposed Likelihood-free BO using likelihood-free inference to extend BO to a broader class of models and utilities. This approach directly models the acquisition function without separately performing inference with a surrogate model. However, there is a risk that the acquisition function becomes overconfident. 

\paragraph{Offline Black-box Optimization.}

Recent studies address BBO in the offline setting using pre-collected datasets without actively querying the oracle function. Existing approaches can be categorized into two types. The first is the forward approach, which aims to improve the likelihood of finding a good design with a surrogate model \citep{yu2021roma, fu2021offline, trabucco2021conservative, dao2024boosting}. The second is the inverse approach, which models a distribution conditioned on the target function score. This distribution can be estimated through adaptive step-size updates in gradient-based optimization via reinforcement learning \citep{chemingui2024offline} or zero-sum games \citep{fannjiang2020autofocused}, or autoregressive modeling \citep{mashkaria2023generative}.
Our work builds on recent progress in inverse approaches, leveraging conditional diffusion models to better learn the data manifold in the design space \citep{krishnamoorthy2023diffusion, kong2024diffusion, li2024diffusion}. In contrast to these recent works, we propose a sample-efficient online BBO algorithm with an acquisition function based on conditional diffusion models.


\paragraph{Diffusion Models.} As an emerging class of generative models with strong expressiveness, diffusion models \citep{sohl2015deep, song2020score} have been successfully deployed across various domains including image generation \citep{rombach2022high}, reinforcement learning \citep{wang2022diffusion}, robotics \citep{chi2023diffusion}, etc. Notably, through the formulation of stochastic differential equations (SDEs), \citep{song2020score} provides a unified continuous-time score-based framework for distinctive classes of diffusion models. 
To steer the generation toward high-quality samples with desired properties, it is important to guide the backward data-generation process using task-specific information.
Hence, different types of guidance are studied in prior works \citep{bansal2023universal, nichol2021glide, zhang2023adding}, including classifier guidance \citep{dhariwal2021diffusion} where the classifier is trained externally, and classifier-free guidance \citep{ho2022classifier}, in which the classifier is implicitly specified. 
In this work, we employ classifier-free guidance to avoid training a separate classifier model, thereby enabling feasible uncertainty quantification in conditional diffusion models.

\paragraph{Uncertainty Quantification.}  
Uncertainty quantification (UQ) often relies on probabilistic modeling, with Bayesian approximation and ensemble learning being two popular types of approaches. 
Bayesian Neural Networks (BNNs) \citep{mackay1992practical,neal2012bayesian,kendall2017uncertainties, zhang2018advances} employ variational inference to sample model weights from a tractable distribution and estimate uncertainty through sample variance. When training large-scale models, 
Monte Carlo dropout \citep{srivastava2014dropout} offers a cost-effective alternative by approximating BNNs during inference \citep{gal2016dropout}.
On the other hand, deep ensembles \citep{lakshminarayanan2017simple} train multiple NNs with different initial weights to gauge uncertainty via model variance. Recent efforts incorporate ensembling techniques in generative models to decompose uncertainty into aleatoric and epistemic components \citep{valdenegro2022deeper, ekmekci2023quantifying}. To further improve the scalability of deep ensembles, \citep{chan2024hyper} proposed hyper-diffusion to quantify the uncertainty with a single diffusion model. In comparison, we take one step further by utilizing the quantified uncertainty of conditional diffusion models to solve the black-box optimization problem as a downstream task.

\section{Preliminaries}
\subsection{Problem Formulation}

Let $f: \mathcal{X} \rightarrow \mathbb{R}$ denote the unknown ground-truth black-box function that evaluates the objective score of any design point $\boldsymbol{x}$, with $\mathcal{X} \subseteq \mathbb{R}^d$. Our goal is to find the optimal design $\boldsymbol{x}^*$ that maximizes $f$:
\begin{equation}     
    \boldsymbol{x}^* \in \argmax\nolimits_{\boldsymbol{x} \in \mathcal{X}} f(\mathbf{\boldsymbol{x}}).
     \label{eqn:obj}
\end{equation}
We are interested in the \textit{online} BBO setting in which $f$ is expensive to evaluate and the number of evaluations is limited. 
With a fixed query budget of $K$ and batch size $N$, 
 we iteratively query $f$ with $N$ new inputs in each batch, and update the surrogate model of $f$ based on observed outputs within $K$ iterations. A key concept in online BBO is the acquisition function, which guides the selection of new design points by balancing exploration and exploitation. This function aims to identify high-performing designs, thereby improving the sample efficiency for online BBO.

 \subsection{Conditional Diffusion Model}

 Diffusion Models \citep{sohl2015deep, song2020score} are probabilistic generative models that learn distributions through a denoising process. These models consist of three components: a forward diffusion process that produces a series of noisy samples by adding Gaussian noise, a reverse process to reconstruct the original data samples from the noise, and a sampling procedure to generate new data samples from the learned distribution. Let the original sample be $\mathbf{x}_0$ and $t$ be the diffusion step. For conditional diffusion models, the conditioning variable $y$ resides in the same objective space as $f(x)$. It is added to both the forward process as $q\left(\boldsymbol{x}_{t} | \boldsymbol{x}_{t-1}, y\right)$ and the reverse process as $p_\theta\left( \boldsymbol{x}_{t-1} \mid \boldsymbol{x}_t, y\right), ~\forall t \in [T]$. The reverse process begins with the standard Gaussian distribution $p(\boldsymbol{x}_T) = \mathcal{N}(\boldsymbol{0}, \boldsymbol{I})$, and denoises $\boldsymbol{x}_t$ to recover $\boldsymbol{x}_0$ through the following Markov chain with reverse transitions:
\begin{align*}
    &p_\theta\left(\boldsymbol{x}_{0: T} | y\right)
 = p(\boldsymbol{x}_T) \prod\nolimits_{t=1}^T p_\theta\left(\boldsymbol{x}_{t-1} \mid \boldsymbol{x}_t, y\right), ~~  \boldsymbol{x}_T \sim \mathcal{N}(\boldsymbol{0}, \boldsymbol{I}),\\
    &p_\theta\left( \boldsymbol{x}_{t-1} \mid \boldsymbol{x}_t, y\right) =\mathcal{N}\left(\boldsymbol{x}_{t-1} ; \mu_\theta(\boldsymbol{x}_t, t, y), \Sigma_\theta (\boldsymbol{x}_t, t, y) \right).
\end{align*}
During training, $\Sigma_\theta$ is empirically fixed, 
and $\mu_\theta$ is reparametrized by a trainable denoise function $\boldsymbol{\epsilon}_\theta\left(\boldsymbol{x}_t, t, y\right)$. This function estimates the noise vector $\epsilon$ added to the input $\boldsymbol{x}_t$, and is trained by minimizing a reweighted version of the evidence lower bound (ELBO):

\begin{equation}        
    \mathcal{L}_{\mathrm{dif}}=\mathbb{E}_{\boldsymbol{x}_0 \sim q(\boldsymbol{x}), y, \boldsymbol{\epsilon} \sim \mathcal{N}(0, \boldsymbol{I}), t \sim \mathcal{U}(0, T), \boldsymbol{x}_t \sim q(\boldsymbol{x}_t \mid \boldsymbol{x}_0, y)}\left[w\left(t\right)\left\|\boldsymbol{\epsilon}-\boldsymbol{\epsilon}_\theta\left(\boldsymbol{x}_t, t, y\right)\right\|_2^2\right].
    \label{eqn:loss}
\end{equation}

Notably, the objective function in \Cref{eqn:loss} \citep{ho2020denoising} resembles denoising score matching across all time steps $t$. It estimates the gradient of the log probability density of the noisy data (i.e. score function): $\boldsymbol{\epsilon}_\theta\left(\boldsymbol{x}_t, t, y\right) \approx - \sigma_t \nabla_{\boldsymbol{x}} \log p(\boldsymbol{x} \mid y)$. We further denote the score function as $s_\theta(\boldsymbol{x}_t, y, t) := -\boldsymbol{\epsilon}_\theta\left(\boldsymbol{x}_t, t, y\right) / \sigma_t$.

\section{\ours{}}
\allowdisplaybreaks

In this section, we present the \ours{} framework, followed by the details of training its conditional diffusion model in the Bayesian setting.

\begin{table}[t]
\centering
\begin{tabular}{c|c|c|c}
\bottomrule[1.5pt]
\textbf{Problem Setting} & \textbf{Surrogate Model Type} & \textbf{Acquisition Function} & \textbf{Method} \\ \midrule
\multirow{2}{*}{\centering Offline BBO} 
    & Forward Model: $p(y|x,\mathcal{D})$ & $\times$ &  \citeauthor{yu2021roma, fu2021offline, dao2024boosting} \\ 
    & Inverse Model: $p(x|y,\mathcal{D})$ & $\times$ & \citeauthor{chemingui2024offline,kong2024diffusion, li2024diffusion}\\ 
    \hline
\multirow{2}{*}{\centering Online BBO} 
    & Forward Model: $p(y|x,\mathcal{D})$ & $\alpha(x,\mathcal{D})$ & \citeauthor{frazier2018tutorial, gruver2024protein, song2022general} \\ 
    & Inverse Model: $p(x|y,\mathcal{D})$ & $\alpha(y,\mathcal{D})$ & Our method: \ours{} \\ 
\bottomrule[1.5pt]
\end{tabular}
\vspace{3mm}
\caption{Overview of black-box optimization settings, highlighting differences between \ours{} and existing methods in terms of problem setting, surrogate modeling type, and acquisition function. Note that offline BBO and online BBO are distinct problem settings. We compare \ours{} only with other online BBO methods. Full discussion of the existing methods is provided in \Cref{sec:related}.}
\label{tab:bbo_comparison}

\end{table}

\begin{algorithm}[b]
\SetAlgoLined
\KwIn{Initial dataset $\mathcal{D} = \{\mathbf{x}, y\}$, total number of iterations $K$, candidate feasible range $C$, oracle function $f(\cdot)$, batch size $N$}
\textbf{Initialization}: Conditional diffusion model $p_\theta(\mathbf{x}|y)$\\
\For{$k =1,2, \cdots K$} {
    Train the conditional diffusion model with $\mathcal{D}$ \\
    Construct a candidate set $\mathcal{Y} = \{y: 0 \leq y \leq C\}$ \\
    $y_k^* = \argmax_{y \in \mathcal{Y}}\alpha(y, \mathcal{D})$ \\
    Generate $\{\mathbf{x}_{j}\}_{j=1}^N$ where $\mathbf{x}_{j} \sim p_\theta(\mathbf{x} \mid y_k^*, \mathcal{D})$ \\
    Query the oracle function $f(\cdot)$ with generated samples $\{\mathbf{x}_{j}\}_{j=1}^N$ \\ 
    $\mathcal{D} \leftarrow \mathcal{D} \cup \{\mathbf{x}_{j}, f(\mathbf{x}_{j})\}_{j=1}^N$ \\
    $\phi_k \leftarrow \max\mathcal(f(x)) ~~s.t.~~x \in \mathcal{D}$ 
}
\KwOut{Reconstructed $\{\phi_k\}_{k=1}^K$}
 \caption{Diff-BBO}
 \label{alg:diff_bo}

\end{algorithm}

\subsection{The \ours{} Framework}

 \ours{} is an online BBO framework that adopts inverse surrogate models. It leverages the conditional diffusion model to learn the conditional distribution of $p(\boldsymbol{x}|y, \mathcal{D})$ with training data $\mathcal{D}$. It supports high-quality conditional sampling given an arbitrary conditioning variable $y$ within the training data distribution. We can design the acquisition function to select $y$ for conditional sampling in the design space. 
\Cref{tab:bbo_comparison} compares forward and inverse approaches for solving online and offline BBO problems. The key distinction is that online BBO requires the design of an acquisition function to actively query new data for sequential decision-making. To the best of our knowledge, \ours{} is the first online BBO approach to propose an acquisition function, $\alpha(y,\mathcal{D})$, for the inverse surrogate model $p(x|y, \mathcal{D})$.
In this setting, the objective in \Cref{eqn:obj} becomes:
\begin{equation}
\label{eqn:optimization_objective}
    \max_{y_k \in \R} \sum\nolimits_{k=1}^K f(\boldsymbol{x}_k), ~~\boldsymbol{x}_k \sim p_{\theta}(\cdot \mid y_k, \mathcal{D}),~~\theta \in \Theta.
\end{equation}

To solve the above optimization problem, we introduce \ours{} in \Cref{alg:diff_bo}. At each iteration $k$, we train a conditional diffusion model and select $y_{k}^*$ with the highest acquisition function score $\alpha(y_k, \mathcal{D})$. In practice, we select $y_k$ from a constructed candidate set $\mathcal{Y}_k$, where we have $y_k = w \cdot \phi_k$. The weight $w$ belongs to a fixed set of positive scalars $\mathcal{W}$ and $\phi_k$ is the maximum function values being queried in the current training dataset $\mathcal{D}$. 
Conditioning on $y_{k}^* = \arg\max_{y_k}{\alpha(y_k, \mathcal{D})}$, we generate $N$ samples $\{\mathbf{x}_{j}\}_{j=1}^N$, where $\boldsymbol{x}_{j} \sim p_\theta(\mathbf{x}|y_k^*, \mathcal{D})$. By querying the black-box oracle to evaluate each $\boldsymbol{x}_j$, we obtain the best possible reconstructed value $\phi_k$ for the current iteration, and append all queried data pairs $\{\mathbf{x}_{j}, f(\mathbf{x}_{j})\}_{j=1}^N$ to the training dataset $\mathcal{D}$. The overall \ours{} framework with the conditional diffusion model is shown in \Cref{fig:framework}.

\subsection{Conditional Diffusion Model Training}

Instead of estimating a set of deterministic parameters $\theta$ from a deterministic neural network, we are interested in learning its Bayesian posterior to further understand and improve the model's performance as well as its reliability with uncertainty quantification. In Bayesian setting, we consider the model parameters $\theta \in \Theta$, where $\Theta$ is the parameter space, and maintain its posterior distribution $p(\theta | \mathcal{D})$, which is learned from training data $\mathcal{D}$. By choosing $\theta$ from its posterior, essentially we sample a score function $\widetilde{s}_\theta(\boldsymbol{x}_t, y, t)$ from the probability distribution $p(s_\theta ~|~\boldsymbol{x}_t, y, t, \mathcal{D}) = \mathcal{N}(s_\theta(\boldsymbol{x}_t, y, t), \Sigma_{s_\theta}(\boldsymbol{x}_t, y, t))$, whose expected value is $s_\theta(\boldsymbol{x}_t, y, t)$, and variance  is a diagonal covariance matrix $\Sigma_{s_\theta}(\boldsymbol{x}_t, y, t)$.

Specifically, we adopt classifier-free guidance as in \citep{ho2022classifier} to eliminate the requirement of training a separate classifier.
We jointly train an unconditional diffusion model $p_\theta(\boldsymbol{x})$ parameterized by $\epsilon_\theta(\mathbf{x}, t, \emptyset)$ and a conditional diffusion model $p_\theta(\boldsymbol{x} | y)$ parameterized by $\epsilon_\theta(\mathbf{x}, t, y)$ 
by minimizing the following loss function:
\begin{equation}
        \mathcal{L}=\mathbb{E}_{\boldsymbol{x}_0, y, \boldsymbol{\epsilon}, t , \boldsymbol{x}_t, \lambda} \left[w\left(t\right)\left\|\boldsymbol{\epsilon}-\boldsymbol{\epsilon}_\theta\left(\boldsymbol{x}_t, t, (1 - \lambda)y  + \lambda \emptyset \right)\right\|_2^2\right],
\end{equation}
where $\boldsymbol{x}_0  \sim q(\boldsymbol{x}), \boldsymbol{\epsilon} \sim \mathcal{N}(0, \boldsymbol{I}), t \sim \mathcal{U}(0, T), \boldsymbol{x}_t \sim q(\boldsymbol{x}_t \mid \boldsymbol{x}_0), \lambda \sim \mathrm{Bernoulli}(p_{\mathrm{uncond}})$, and $p_{\mathrm{uncond}}$ is the probability of setting $y$ to the unconditional information $\emptyset$.

\begin{figure*}[t]
    \centering
    \includegraphics[width=0.9\linewidth]{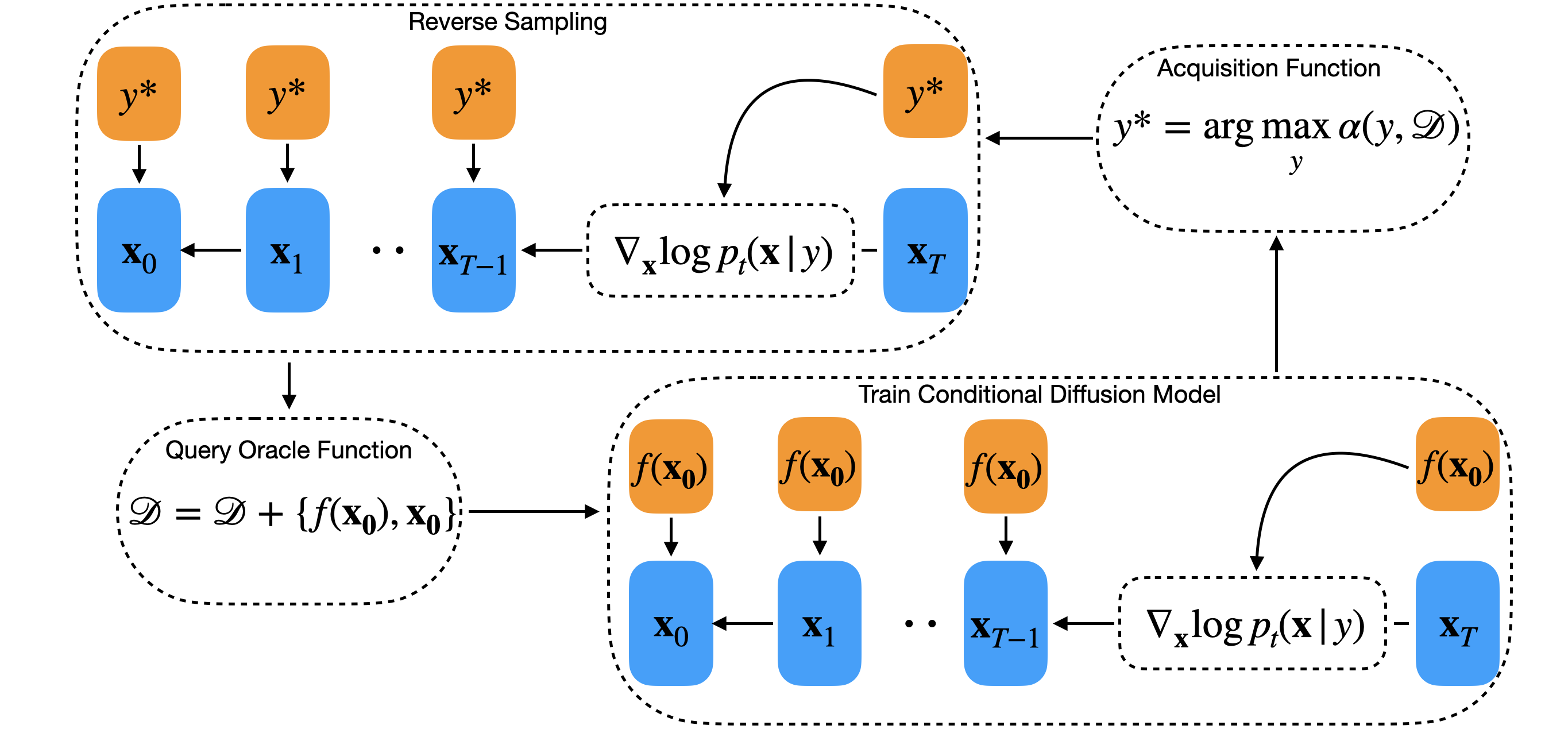}
    \caption{\ours{} framework using the conditional diffusion model as the inverse surrogate model. It includes $4$ stages: 1. Train the conditional diffusion model given the training dataset. 2. Compute the acquisition function and select the optimal $y^*$ to condition on. 3. Generate samples $\{\mathbf{x_0}\}$ conditioned on $y^*$. 4. Query the oracle given generated samples $\{\mathbf{x_0}\}$ and update the training dataset.}
    \label{fig:framework}
\end{figure*}

\section{Acquisition Function Design}
In this section, we propose a novel acquisition function called \textit{Uncertainty-aware Exploration} (UaE) for \ours{}. We first analyze the uncertainty of the conditional diffusion model from both theoretical and practical perspectives, decomposing the uncertainty into the aleatoric and epistemic components. Based on the uncertainty decomposition, we then propose the acquisition function UaE. We prove that by achieving a balance between high objective values and low epistemic uncertainty, UaE effectively provides a near-optimal solution for online BBO. 

\subsection{Uncertainty Quantification on Conditional Diffusion Model}
\label{sec:acqui_uq}

The optimization problem defined in \Cref{eqn:optimization_objective} presents a probabilistic formulation of the online BBO problem using inverse modeling. Instead of searching for a single optimal point $\boldsymbol{x}$, it aims to learn a parameterized distribution $p_\theta(\boldsymbol{x} \mid y, \mathcal{D})$ for a given $y$, and sample from this predictive distribution. As such, we resort to the tools of Bayesian inference to solve this task. More specifically, given an observed value $y$ of a sample $\boldsymbol{x}$, the objective of Bayesian inference is to estimate the predictive distribution:
\begin{align}
\label{eqn:pred_distn}
    p(\boldsymbol{x} \mid y, \mathcal{D}) = \E_{\theta}[p_\theta(\boldsymbol{x} \mid y)] = \int_\theta p_\theta(\boldsymbol{x} \mid y) p(\theta \mid \mathcal{D}) d \theta.
\end{align}
Its empirical estimation over an ensemble of $M$ conditional diffusion models is computed as:
\begin{align*}
    \widehat{\E}_{\theta}[p_\theta(\boldsymbol{x} \mid y)] = \frac{1}{M} \sum\nolimits_{i=1}^M p_{\theta_i}(\boldsymbol{x} \mid y).
\end{align*}

By \Cref{eqn:pred_distn}, we recognize that the uncertainty arises from two sources: uncertainty in deciding parameter $\theta$ from its posterior $p(\theta | \mathcal{D})$ and uncertainty in generating sample $\boldsymbol{x}$ from a fixed diffusion model $p_\theta(\boldsymbol{x} \mid y)$ after $\theta$ is chosen. 
Before proceeding with the uncertainty decomposition in the conditional diffusion model, it is crucial to understand how to capture the overall uncertainty when using a diffusion model to generate $\boldsymbol{x}$.
Essentially, it can be explicitly traced through the denoising process. More specifically, \Cref{thm:sde_uncertainty} provides analytical solutions to compute the uncertainty on a single denoising process of general score-based conditional diffusional models. It offers theoretical insights of how uncertainty is being propagated through the reverse denoising process both in discrete time and continuous time, which is characterized through the lens of stochastic differential equations (SDEs) of the Ornstein–Uhlenbeck (OU) process. Detailed proofs can be found in \Cref{sec:app_uq}. 
\begin{restatable}{theorem}{uncertaintyPropagation} (Uncertainty propagation)
\label{thm:sde_uncertainty}
    Let $t\in[T]$ be the diffusion step, $s_\theta(\boldsymbol{x}, y , t)$ be the score function of the corresponding diffusion model $p_{\theta}(\boldsymbol{x} \mid y)$. For a single conditional diffusional model $p_{\theta}(\boldsymbol{x} \mid y)$, the uncertainty in generating a sample $\boldsymbol{x}$ can be analytically traced through the discrete-time reverse denoising process as follows:
    \begin{equation*}
        \mathrm{Var}(\boldsymbol{x}_{t-1}) 
    = \frac{1}{4} \mathrm{Var}(\boldsymbol{x}_{t}) +   \mathrm{Var}(s_\theta(\boldsymbol{x}, y , t)) + \frac{1}{2} \left( \E\left[\boldsymbol{x}_{t} \circ s_\theta(\boldsymbol{x}_t, y , t)\right] - \E[\boldsymbol{x}_{t}]\circ\E[s_\theta(\boldsymbol{x}_{t}, y , t)] 
    \right)
    + I,
    \end{equation*}
    \begin{equation*}
        \E(\boldsymbol{x}_{t-1}) 
    = \frac{1}{2} \E(\boldsymbol{x}_{t}) + \E (s_\theta(\boldsymbol{x}, y , t)),
    \end{equation*}
    where $\circ$ is the Hadamard product, and $I$ is the identity matrix. Similarly, in continuous-time process, the uncertainty can be captured as follows:
    \begin{align}
        \mathrm{Var}(\boldsymbol{x}_0) = (T + 1)I + \mathrm{Var}\left( \int_{t=0}^T \left( \frac{1}{2} \boldsymbol{x}_t + s_\theta(\boldsymbol{x}, y , t) \right) \dif t \right).
    \end{align}

\end{restatable}
While \Cref{thm:sde_uncertainty} establishes the existence of closed-form solutions to quantify uncertainty based on the intrinsic properties of diffusion models, performing exact Bayesian inference when training diffusion models in practice requires non-trivial efforts and can be computationally demanding. Hence, in \Cref{sec:ud}, we will introduce a practically-efficient approach to quantify and decompose the uncertainty based on \Cref{eqn:pred_distn} .

\subsection{Uncertainty Decomposition}
\label{sec:ud}

To systematically analyze the effect of uncertainty in our inverse modeling approach using the conditional diffusion model, we now provide a practical method to perform uncertainty decomposition in terms of the aleatoric component and its epistemic counterpart. 

The \textit{aleatoric uncertainty} is captured by the variance of the likelihood $p_{\theta}(\boldsymbol{x} \mid y)$, which is proportional to the variance of the measurement noise during sample generation, irreducible and task-inherent. To estimate the aleatoric uncertainty, we can Monte Carlo (MC) sample $\boldsymbol{x}$ for $N$ times from a learned likelihood function $p_\theta(\boldsymbol{x} \mid y)$ for fixed $y, \theta$.

In contrast, the \textit{epistemic uncertainty} is captured through the variance of the posterior distribution $p(\theta \mid \mathcal{D})$, which is proportional to the variance of the score network, and is reducible with the increase of training data. Recall that $\Theta$ is the parameter space that contains all possible model parameters $\theta$, which are used to generate samples from the predictive distribution $p(\boldsymbol{x} \mid y, \mathcal{D})$. As the dataset size and quality grows, the variance of 
the posterior distribution shrinks, corresponding to the reduction of epistemic uncertainty in learned parameters $\theta \sim p(\theta \mid \mathcal{D})$.
To estimate the epistemic uncertainty, we use the ensemble approach. 
We train $M$ conditional diffusion models by initializing them with different random seeds and obtain $M$ model parameters $\{\theta_i\}_{i=1}^M$. 
Then we generate $N$ samples $\{\boldsymbol{x}_j\}_{j=1}^N$ for each diffusion model with corresponding parameter $\theta_i,~\forall i \in [M]$. Combining the above gives a systematic way to decompose and estimate the two types of uncertainty in practice, which is formally described in \Cref{thm:decompose_un}.

\begin{proposition} [Uncertainty Decomposition]
\label{thm:decompose_un}
At each iteration $k \in [K], \forall i \in [M], j \in [N]$, the overall uncertainty in inverse modeling can be decomposed into its aleatoric and epistemic components, which can be empirically measured as follows:

\begin{equation}
\begin{aligned} 
    \Delta_{\text {aleatoric }} (y, \mathcal{D}) 
    &=\mathbb{E}_{\theta_i \sim p(\cdot \mid \mathcal{D})}\left[\operatorname{Var}_{\boldsymbol{x}_{i,j} \sim p_{\theta_i}(\cdot \mid y)}\left( \|\boldsymbol{x}_{i, j} \|
    \right)\right];\\ 
    \Delta_{\text {epistemic }} (y, \mathcal{D})
    &=\operatorname{Var}_{\theta_i \sim p(\cdot \mid \mathcal{D})}\left(\mathbb{E}_{\boldsymbol{x}_{i,j} \sim p_{\theta_i}(\cdot \mid y)}\left[ \| \boldsymbol{x}_{i, j} \|
    \right]\right).
\end{aligned}
\end{equation}

\end{proposition}

\subsection{Uncertainty-aware Exploration}
\label{sec:acqui_design}
At each iteration $k \in [K]$ of \ours{} algorithm, the acquisition function $\alpha(y_k, \mathcal{D})$ proposes an optimal scalar value $y_k^*$ as follows: 
$$
    y_k^* = \operatorname{argmax}_{y_k} \alpha(y_k, \mathcal{D}),
$$
which is used to generate $\boldsymbol{x}$ in the design space using conditional diffussion model.

Intuitively, designing an effective acquisition function for inverse modeling requires a balance between high objective values $y_k$ and low epistemic uncertainty of $y_k$.
On the one hand, it is advantageous to generate samples $\boldsymbol{x}$ conditioned on higher $y_k$ and query the oracle with these promising samples for online BBO. 
On the other hand, the epistemic uncertainty of $y_k$ can be used to gauge the approximation error between $y_{k}^*$ and the reconstructed function value $\max_{j \in [N]}  f(\boldsymbol{x}_j)$, where $f(\cdot)$ is the black-box oracle, and $\boldsymbol{x}_j \sim p_{\theta}(\cdot|y_k^*, \mathcal{D}), \forall j \in [N]$.

Therefore, we propose the acquisition function named \textit{Uncertainty-aware Exploration} (UaE):
\begin{equation}
\label{eqn:acquisition}
    \alpha(y, \mathcal{D}) = y - \Delta_{\text {epistemic}}(y, \mathcal{D}),
\end{equation}
which utilizes the uncertainty estimation on conditional diffusion model as given in \Cref{thm:decompose_un}. It effectively penalizes the candidates for which the model is less certain. As shown later, by balancing the exploration-exploitation trade-off, UaE provides an effective way to solve the online BBO problem.

\subsection{Sub-optimality of UaE}
\label{sec:acqui_design1}
To evaluate the quality of generated samples, we theoretically analyze the sub-optimality performance gap between $y_k^*$ and reconstructed value at each iteration. In particular, \Cref{thm:y_bound} and \Cref{thm:sub_optimality_gap} demonstrate that such sub-optimality gap can be effectively handled in inverse modeling, contributing to the optimality of our proposed algorithm as established in \Cref{thm:optimization_equivalence}. We emphasize that the purpose of \Cref{thm:y_bound} and \Cref{thm:sub_optimality_gap} is to support \Cref{thm:optimization_equivalence} by validating the design of UaE. These results are designed to hold universally with probability one, independent of hyperparameters such as batch size. As a result, the bounds presented are intentionally conservative and are not tightened to rely on high-probability guarantees.  Detailed proofs are deferred to \Cref{sec:app_sub_optimality}.

We first show that by using conditional diffusion model, the expected error of the sub-optimality performance gap can be effectively bounded under mild assumptions.

\begin{restatable}{theorem}{yBoundExp} 
\label{thm:y_bound}
    At each iteration $k \in [K]$, define the sub-optimality performance gap as 
    \begin{equation}
    \label{eqn:sub_optimality_gap_main}
        \Delta(p_\theta, y_k^*) = \left| y_k^*- \max_{j \in [N]} f(\boldsymbol{x}_j) \right|, ~~ \mathrm{where}~~\boldsymbol{x}_j \sim p_{\theta}(\cdot|y_k^*, \mathcal{D}).
    \end{equation}
    Assume that there exists some $\theta^* \sim p(\theta | \mathcal{D})$ that produces a probability distribution  $p_{\theta^*}(\cdot \mid \mathcal{D})$ such that it is able to generate a sample $\boldsymbol{x}^*$ that perfectly reconstructs $y_k^*$. Suppose function $f$ is $L$-Lipschitz and each sample is $\sigma$-subGaussian, 
    it can be shown that 
    \begin{equation*}
        \E\left[{\Delta(p_\theta, y_k^*)}\right] 
        \leq c_1 L \sqrt{d} \sigma,
    \end{equation*}
    where $d$ is the dimensionality of the design space, $c_1$ is some universal constant.
\end{restatable}

\Cref{thm:y_bound} suggests that in expectation, the reconstructed function value 
closely approximates the provided conditional information $y_k^*$, implying \ours{} is effective in searching for promising samples in the design space by utilizing the information from the objective space. Hence, to achieve a robust estimator for the online BBO problem, the primary concern shifts to controlling the variance of the sub-optimality gap defined in  \Cref{eqn:sub_optimality_gap_main}, which is further assessed and evaluated in \Cref{thm:sub_optimality_gap}.

\begin{restatable}{theorem}{yBoundVar} \textrm(Sub-optimality bound)
\label{thm:sub_optimality_gap}
    At each iteration $k \in [K]$, suppose $M$ model parameters $\{\theta_i\}_{i=1}^M$ are generated from the ensemble model for some fixed dataset $\mathcal{D}$. Suppose function $f$ is $L$-Lipschitz, 
    it can be shown that the variance of the sub-optimality performance gap of each model is bounded by the epidemic uncertainty:
    \begin{equation}   
    \mathrm{Var}\left(\Delta(p_{\theta_i}, y_k^*)\right) \leq c_2 L^2 d \sigma^2 + c_2 L^2 \Delta_{\mathrm{epistemic}}(y_k^*, \mathcal{D}),
    \end{equation}
    where $c_2$ is some universal positive constant.
\end{restatable}
\Cref{thm:sub_optimality_gap} shows that the variance of the sub-optimality performance gap is upper bounded by the epistemic uncertainty of diffusion model with some global constants. It implies that decreasing the epistemic uncertainty reduces the variance of the performance gap, leading to more reliable results. Therefore, it is crucial to achieve an effective balance between maximizing the objective value and minimizing the corresponding epistemic uncertainty when designing the acquisition function. By doing so, UaE not only explores the objective space with high-value solutions, but also ensures stability and consistency of the optimization process. 


Finally, we prove in \Cref{thm:optimization_equivalence} that by adopting UaE for inverse modeling to guide the selection of generated samples for solving BBO problems, we can obtain a near-optimal solution for the online optimization problem defined in \Cref{eqn:optimization_objective}. The proof is available in \Cref{sec:app_optimization}. 

\begin{restatable}{theorem}{optiEqn}
\label{thm:optimization_equivalence}
    Let $\mathcal{Y}$ be the constructed candidate set at each iteration $k \in [K]$ in \Cref{alg:diff_bo}. By adopting UaE as the acquisition function to guide the sample generation process in conditional diffusion model, \ours{} (\Cref{alg:diff_bo}) achieves a near-optimal solution for the online BBO problem defined in \Cref{eqn:optimization_objective}:

    \begin{align*}
        \qquad\max_{y_k \in \R} \sum\nolimits_{k=1}^K f(\boldsymbol{x}_k), ~~\boldsymbol{x}_k \sim p_{\theta}(\cdot \mid y_k, \mathcal{D}),~~\theta \in \Theta ~~\Rightarrow \max_{y_k \in \mathcal{Y}} \sum\nolimits_{k=1}^K \alpha(y_k, \mathcal{D}).
    \end{align*}

\end{restatable}

Hence, equipped with the novel design of UaE, \ours{} is a theoretically sound approach utilizing inverse modeling to effectively solve the online BBO problem.

\section{Experiments}
\begin{figure*}[b]
    \centering
    \includegraphics[width=\linewidth,trim={30, 30, 30, 30}]{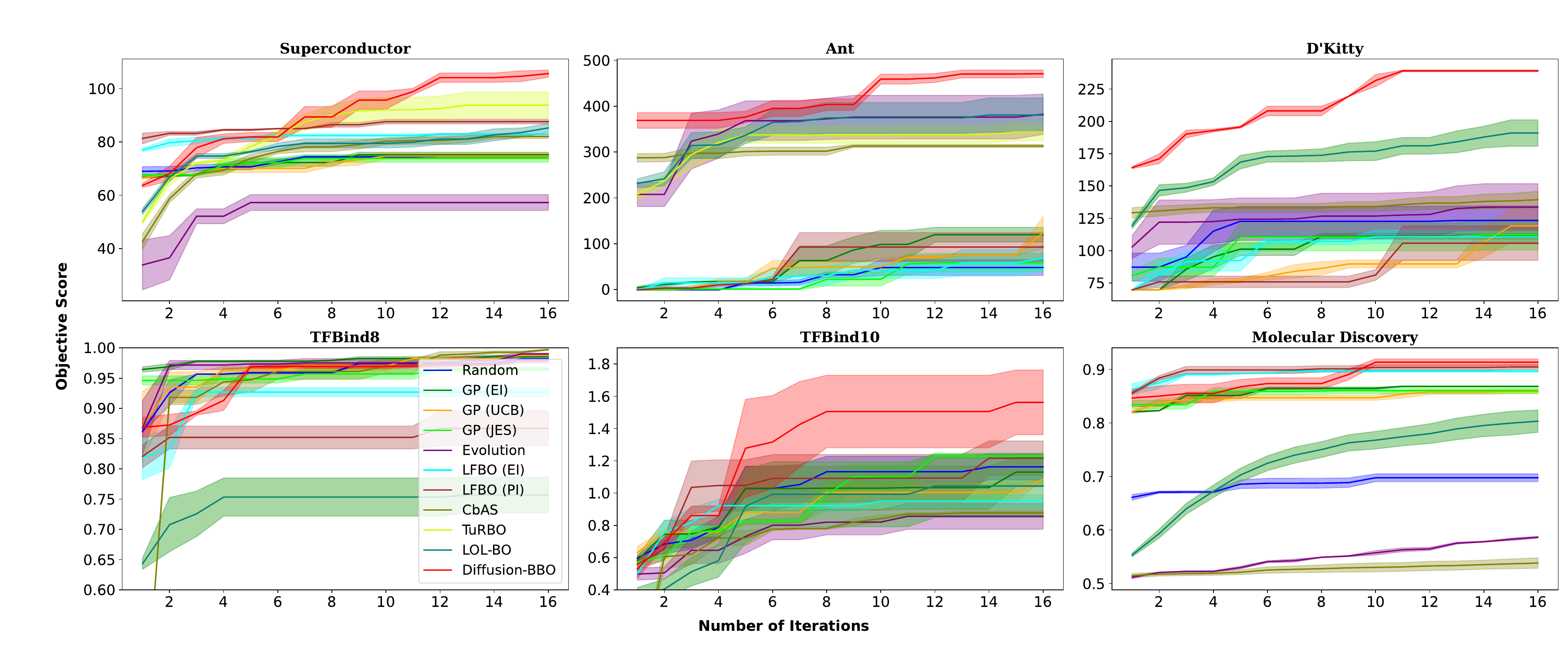}
    \caption{Comparison of \ours{} with baselines for online black-box optimization on DesignBench and Molecular Discovery task. All plots start at iteration 1 after one round of data queries. We plot the mean values and the confidence interval based on three random runs. \ours{} exhibits superior performance with respect to sample efficiency.}
    \label{fig:online_opt}
\end{figure*}

To validate the efficacy of \ours{}, we conduct experiments on six online black-box optimization tasks for both continuous and discrete optimization settings. 
We arrange the pre-collected data in ascending order based on objective values and select data from the 25th to the 50th percentile as the initial training dataset. The goal is to demonstrate that low-quality data, with below-average objective function values, is sufficient for the initial training stage of \ours{}.
Each optimization iteration is allocated $100$ queries to the oracle function (batch size $N=100$), with a total of $16$ iterations conducted. 
For the fixed set of weights, we use $\mathcal{W}=\{0.6,0.7,0.8,0.9,1.0\}$. The corresponding candidate set then becomes $\mathcal{Y}_k = \{y_k = w \cdot \phi_k \mid w \in \mathcal{W}\}$.
 The rationale behind this choice is that $w > 1$
 falls outside the training data distribution, which we also verify through the ablation study. For UaE computation, we take the logarithm of both terms to ensure they are within the same range.
More details of the experimental setup are provided in \Cref{sec:app_experiment}.

\subsection{Dataset}
\label{subsection:dataset}

We restructured $5$ high-dimensional scientific discovery tasks from Design-Bench to facilitate online black-box optimization. We test on $3$ continuous and $2$ discrete tasks. In \textbf{D’Kitty} and \textbf{Ant} Morphology, the goal is to optimize for the morphology of robots. In \textbf{Superconductor}, the aim is to optimize for finding a superconducting material with a high critical temperature. \textbf{TFBind8} and \textbf{TFBind10} are discrete tasks where the goal is to find a DNA sequence that has a maximum affinity to bind with a specified transcription factor. 
We also include the \textbf{Molecular Discovery} task to optimize a compound’s activity against a biological target with therapeutic value. 
More details of the dataset are provided in \Cref{sec:app_dataset}.

\begin{wrapfigure}{r}{0.45\textwidth}
    \centering
    \vspace{-10pt}
    \includegraphics[width=0.91\linewidth]{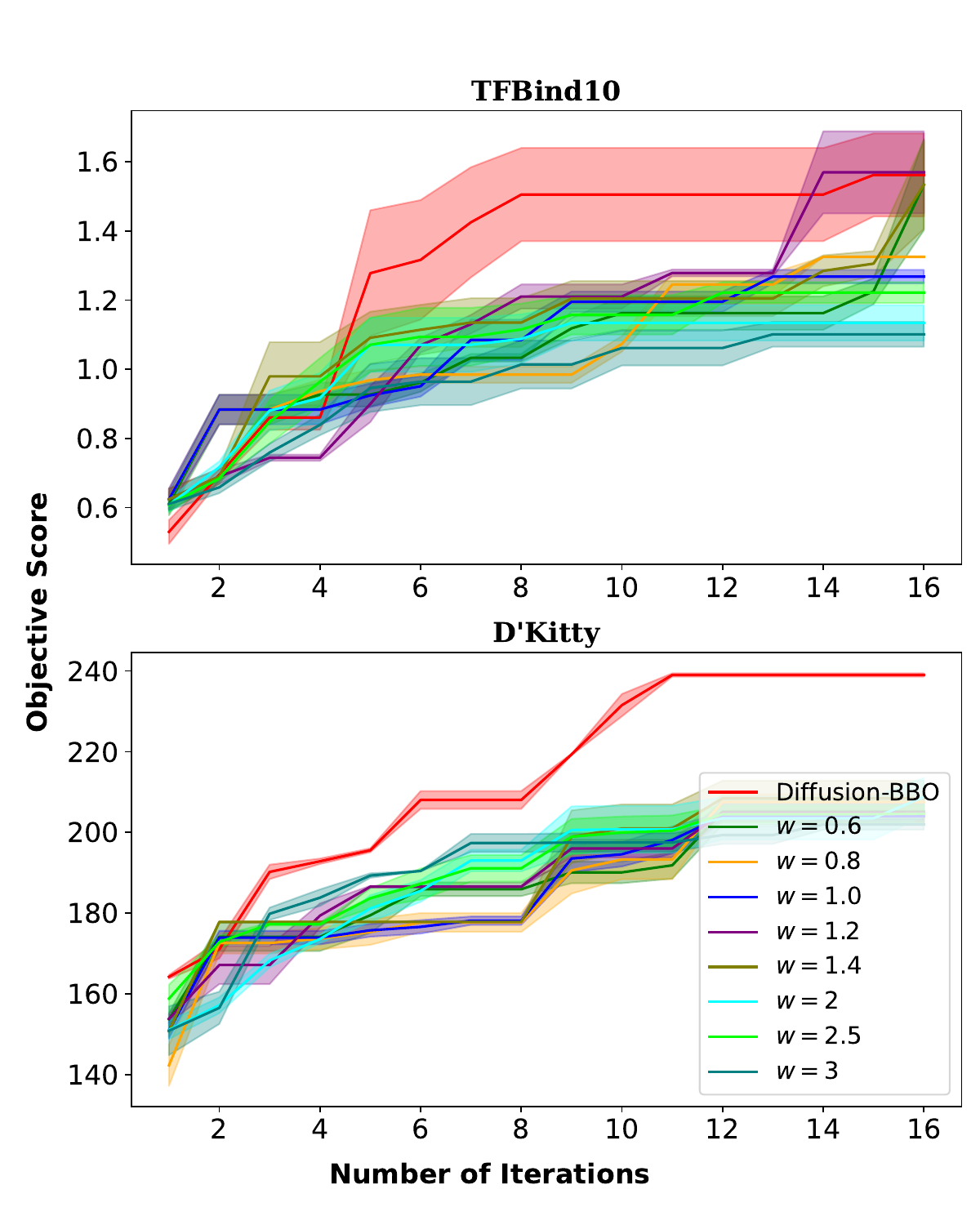}
    \caption{Impact of acquisition function design for black-box optimization on both discrete task (\textbf{TFBind10}) and continous task (\textbf{D’Kitty}). Comparison of \ours{} with UaE against the fixed conditioning approaches using weights $w \in \{0.6, 0.8, 1.0, 1.2, 1.4, 2.0, 2.5, 3.0\}$. Results averaged across three random runs.}
    \vspace{-10pt}
    \label{fig:ablation1}
\end{wrapfigure}

\subsection{Baselines}

We compare \ours{} with 10 baselines, including Bayesian optimization (BO), trust region BO (TuRBO) \citep{eriksson2019scalable}, local latent space Bayesian optimization (LOL-BO) \citep{maus2022local}, likelihood-free BO (LFBO) \citep{song2022general}, evolutionary algorithms \citep{brindle1980genetic, real2019regularized}, conditioning by adaptive sampling (CbAS) \citep{brookes2019conditioning}, and random sampling. For BO approaches, we include Gaussian Processes (GPs) with Monte Carlo (MC)-based batch expected improvement (EI), MC-based batch upper confidence bound (UCB) \citep{wilson2017reparameterization}, and joint entropy search (JES \citep{hvarfner2022joint} as the acquisition functions. For LFBO, we use EI and probability of improvement (PI) as the acquisition functions.

\subsection{Results}

Figure \ref{fig:online_opt} illustrates the performance across six datasets for all baselines and our proposed \ours{}. Notably, \ours{} utilizing UaE as the acquisition function consistently outperforms other baselines in both discrete and continuous settings, with the sole exception of the TF-BIND-8 task. These results align with our theoretical analysis, demonstrating the effectiveness of the designed UaE. Additionally, in the Ant and Dkitty tasks, \ours{} demonstrates a significant lead over all baseline methods, starting from the very first iteration of the online optimization process. This remarkable performance can be attributed to the conditional diffusion model of \ours{}, which can effectively learn the data manifold in the design space from the low-quality initial training dataset.
In contrast, the forward approach employed by BO and LFBO, which relies solely on optimizing the forward surrogate model, is more prone to converging on suboptimal solutions.

\subsection{Ablation study}
\label{subsection:ablation}

In this section, we conduct ablation studies to investigate the impact of our designed acquisition function, UaE. We compare \ours{} with the fixed conditioning approach for data querying. 
Instead of using UaE to dynamically select $w$ and the corresponding $y_k = w \cdot \phi_k$ for conditional sampling, the fixed conditioning approach predetermines a fixed weight $w$ and consistently generates new samples conditioned on $w \cdot \phi_k$.
As shown in Figure \ref{fig:ablation1}, \ours{} with UaE consistently outperforms the fixed conditioning approach. Furthermore, it can be found that simply conditioning on higher $w \cdot \phi_k$ by applying $w > 1$ does not enhance optimization performance. This highlights the effectiveness of UaE in identifying the optimal $y_k$ for conditioning by balancing between targeting higher objective values and minimizing the corresponding epistemic uncertainty.

Additionally, we analyze the computational time for model training and acquisition function computation for \ours{} and existing baselines, as shown in \Cref{tab:time_analysis}. The results indicate that the computational time for \ours{} is comparable to BO approaches using GP as the surrogate model, typically ranging from 1 to 2 minutes per iteration for model training and acquisition function computation. Given the context of online BBO, where querying the oracle to generate new data is the most expensive or time-consuming part, a few minutes spent on training and acquisition function computation should not be considered a significant computational burden.

Furthermore, we evaluate the effect of batch size $N$ (i.e., the number of queries per iteration) on \ours{}. The results can be found in Appendix \ref{section:ablation_appendices}. 

\begin{table*}[h]
\small
\centering
\resizebox{\textwidth}{!}{
\begin{tabular}{c|cccccccc}
\bottomrule[1.5pt]
\textbf{Task} & GP (EI) & GP (UCB) & Evolution & LFBO (EI) & LFBO (PI) & CbAS & TuRBO & \ours{}\\  \hline
TFBind8 & $112.92$ & $113.81$ & $0.0021$ & $0.59$ & $1.44$ & $0.075$ & $67.23$ & $136.29$ \\ 
Molecular Discovery & $53.14$ & $53.82$ & $0.0024$ & $1.93$ & $1.12$ & $0.023$ & $76.84$ & $69.44$ \\ 
\bottomrule[1.5pt]
\end{tabular}
}
\caption{Model training and acquisition function computation time in seconds.}
\label{tab:time_analysis}
\end{table*}

\section{Discussion \& Conclusion}
\label{section:conclude}

In this paper, we introduced \ours{}, a sample-efficient online BBO framework leveraging the conditional diffusion model as the inverse surrogate model.
By utilizing the novel acquisition function UaE, \ours{} strategically proposes the conditioning variable $y$ in the objective space to improve the sample efficiency. Our empirical evaluations on $6$ real-world scientific discovery tasks demonstrate that \ours{} achieves state-of-the-art performance, establishing its potential as a robust tool for efficient online black-box optimization. Theoretically, we prove that using UaE leads to optimal optimization solutions.
For future work, \ours{} can be extended to various Bayesian optimization settings, including multi-objective and multi-fidelity optimization, to develop an adaptive experimental design platform for scientific discovery tasks.



\section*{Acknowledgement}

This work was supported in part  by the U.S. Army Research Office
under Army-ECASE award W911NF-23-1-0231, the U.S. Department Of Energy, Office of Science, IARPA HAYSTAC Program, CDC-RFA-FT-23-0069, DARPA AIE FoundSci, NSF Grants \#2205093, \#2146343, \#2134274, \#2134209, and \#2112665 (TILOS).

\bibliography{ref}
\bibliographystyle{unsrtnat}

\newpage
\appendix
\onecolumn
\textbf{\Large{Appendices}}
\counterwithin{lemma}{section}

\section{Uncertainty Quantification Through SDEs}
\label{sec:app_uq}
\subsection{Conditional Diffusion SDE}
It can be shown that the conditional diffusion model can be represented by the Ornstein–Uhlenbeck (OU) process, which is a time-homogeneous continuous-time Markov process:
\begin{equation}
    \dif \boldsymbol{x}_t = - \gamma \boldsymbol{x}_t \dif t + \sigma \dif \boldsymbol{w}_t,
\end{equation}
where $\gamma$ is the relaxation rate, $\sigma$ is the strength of fluctuation, and $\boldsymbol{w}_t$ is the standard Wiener process (a.k.a., Brownian motion).  Both $\gamma$ and $\sigma$ are time-invariant. In particular, setting $\gamma = 1$ and $\sigma = \sqrt{2}$, we are able to establish that Denoising Diffusion Probabilistic Model (DDPM) is equivalent to OU process observed at discrete times. In the remaining text, we consider SDEs for general score-based diffusion models. The SDE of the forward process in conditional diffusion model can then be written as:
\begin{equation}
\label{eqn:sde_fordward}
    \dif \boldsymbol{x}_t = - \frac{1}{2} g(t) \boldsymbol{x}_t \dif t + \sqrt{g(t)} \dif \boldsymbol{w}_t, ~~ \boldsymbol{x}_0 \sim q( \boldsymbol{x} | y)
\end{equation}
where $g(t)$ is a nondecreasing weighting function that controls the speed of diffusion in the forward process and $g(t) > 0$. For simplicity of analysis, we fix $g(t) = 1$ for all $t \in [T]$.

The generation process of a conditional score-based diffusion model can be viewed as a particular discretization of the following
reverse-time SDE: 
\begin{equation}
\label{eqn:sde_backward}
    \dif \boldsymbol{x}_t = \left( \frac{1}{2} \boldsymbol{x}_t - \nabla_{\boldsymbol{x}_t} \log p(\boldsymbol{x}_t | y) \right) \dif t  + \dif \boldsymbol{w}_t, ~~ \boldsymbol{x}_0 \sim p(\boldsymbol{x}_T | y).
\end{equation}
In practice, the unknown ground truth conditional score $\nabla_{\boldsymbol{x}_t} \log p(\boldsymbol{x}_t | y)$ needs to be estimated with score networks. Let such estimator denoted by $s_\theta(\boldsymbol{x}, y, t)$, then the conditional sample generation is to simulate the following backward SDE:
\begin{equation}
\label{eqn:sde_backward_approx}
    \dif \boldsymbol{x}_t = \left( \frac{1}{2} \boldsymbol{x}_t - s_\theta(\boldsymbol{x}, y , t) \right) \dif t  + \dif \boldsymbol{w}_t, ~~ \boldsymbol{x}_0 \sim \mathcal{N}(\boldsymbol{0}, \boldsymbol{I}).
\end{equation}

In Bayesian settings, we sample a score function $\widetilde{s}_\theta(\boldsymbol{x}_t, y, t)$ from the probability distribution $p(s_\theta | \boldsymbol{x}_t, y, t, \mathcal{D}) = \mathcal{N}(s_\theta(\boldsymbol{x}_t, y, t), \Sigma_\theta(\boldsymbol{x}_t, y, t))$ with expected value $s_\theta(\boldsymbol{x}_t, y, t)$, and diagonal covariance $\Sigma_\theta(\boldsymbol{x}_t, y, t)$.


\subsection{Estimation of Uncertainty}
In this section, we quantify the uncertainty of a single conditional diffusion model in both discrete-time and continuous-time reverse process for \Cref{thm:sde_uncertainty}.

\subsubsection{Uncertainty in Discrete-time
Reverse Process}
We first proof the first statement of \Cref{thm:sde_uncertainty}. We consider the Euler discretization of \Cref{eqn:sde_backward_approx}, which leads to:
\begin{equation}
    \boldsymbol{x}_{t-1} = \frac{1}{2}\boldsymbol{x}_t + s_\theta(\boldsymbol{x}, y , t) + \boldsymbol{\epsilon}, ~~ \epsilon \sim \mathcal{N}(\boldsymbol{0}, \boldsymbol{I}).
\end{equation}
We thus have,
\begin{align}
    \mathrm{Var}(\boldsymbol{x}_{t-1}) 
    &= \frac{1}{4} \mathrm{Var}(\boldsymbol{x}_{t}) +   \mathrm{Var}(s_\theta(\boldsymbol{x}, y , t)) + \frac{1}{2} \mathrm{Cov}\left(\boldsymbol{x}_{t}, s_\theta(\boldsymbol{x}, y , t)\right) + I. \label{eqn:var_discrete}\\
    \E(\boldsymbol{x}_{t-1}) 
    &= \frac{1}{2} \E(\boldsymbol{x}_{t}) + \E (s_\theta(\boldsymbol{x}, y , t)). \label{eqn:exp_discrete}
\end{align}
Here $\mathrm{Cov}\left(\boldsymbol{x}_{t}, s_\theta(\boldsymbol{x}, y , t)\right)$ is the element-vise covariance between $\boldsymbol{x}_{t}$ and $s_\theta(\boldsymbol{x}, y , t)$. Note that we only need to consider the correlation between $\boldsymbol{x}_{t}$ and $s_\theta(\boldsymbol{x}, y , t)$ at the same time step. As a result, to estimate $\mathrm{Cov}\left(\boldsymbol{x}_{t}, s_\theta(\boldsymbol{x}, y , t)\right)$, we have,
\begin{align*}
    \mathrm{Cov}\left(\boldsymbol{x}_{t}, s_\theta(\boldsymbol{x}, y , t)\right) 
    &= \E\left[\left(\boldsymbol{x}_{t} - \E[\boldsymbol{x}_{t}]\right) \left(s_\theta(\boldsymbol{x}, y , t) - \E[s_\theta(\boldsymbol{x}, y , t)]\right)^\rT\right] \\
    &= \E\left[\boldsymbol{x}_{t} \circ s_\theta(\boldsymbol{x}, y , t)\right] - \E[\boldsymbol{x}_{t}]\circ\E[s_\theta(\boldsymbol{x}, y , t)] \\
    &= \E_{\boldsymbol{x}_{t}}\left[\boldsymbol{x}_{t} \circ s_\theta(\boldsymbol{x}, y , t)\right] - \E[\boldsymbol{x}_{t}]\circ\E_{\boldsymbol{x}_{t}}[s_\theta(\boldsymbol{x}_{t}, y , t)] \numberthis \label{eqn:discrete_cov}
\end{align*}
where $\circ$ is the Hadamard product and the third equality is by tower's rule. Substituting \Cref{eqn:discrete_cov} back to \Cref{eqn:var_discrete} completes the proof of the first part of \Cref{thm:sde_uncertainty}.

\subsubsection{Uncertainty in Continuous-time Reverse Process}
We now proof the second statement of \Cref{thm:sde_uncertainty}. To perform the uncertainty quantification for the continuous-time reverse process, we posit the following assumption.

\begin{assumption}
\label{ass:contin_sde}
    For valid $t \in [0,T]$, the generating process $\boldsymbol{x}_t$ in \Cref{eqn:sde_backward} is integrable and has finite second-order moments.
\end{assumption}

With \Cref{ass:contin_sde}, integrating \Cref{eqn:sde_backward} with respect to $t$ yields:
\begin{align}
\label{eqn:continuous_integ}
    \boldsymbol{x}_0 = \boldsymbol{x}_T - \int_{t=0}^T \left( \frac{1}{2} \boldsymbol{x}_t + \nabla_{\boldsymbol{x}_t} \log p(\boldsymbol{x}_t | y) \right) \dif t + \int_{t=0}^T \dif \boldsymbol{w}_t.
\end{align}
Applying the variance operator to both sides of  
\begin{align*}
    \mathrm{Var}(\boldsymbol{x}_0) 
    &= \mathrm{Var}(\boldsymbol{x}_T) + \mathrm{Var}\left( \int_{t=0}^T \left( \frac{1}{2} \boldsymbol{x}_t + \nabla_{\boldsymbol{x}_t} \log p(\boldsymbol{x}_t | y) \right) \dif t \right) + \mathrm{Var} \left(\int_{t=0}^T \dif \boldsymbol{w}_t \right) \\
    &= I + \mathrm{Var}\left( \int_{t=0}^T \left( \frac{1}{2} \boldsymbol{x}_t + \nabla_{\boldsymbol{x}_t} \log p(\boldsymbol{x}_t | y) \right) \dif t \right) + \E \left[ \left(\int_{t=0}^T \dif \boldsymbol{w}_t \right)^2 \right] - \left(\E \left[ \int_{t=0}^T \dif \boldsymbol{w}_t \right]\right)^2 \\
    &= (T + 1)I + \underbrace{\mathrm{Var}\left( \int_{t=0}^T \left( \frac{1}{2} \boldsymbol{x}_t + \nabla_{\boldsymbol{x}_t} \log p(\boldsymbol{x}_t | y) \right) \dif t \right)}_{V_1} \numberthis \label{eqn:var_x0},
\end{align*}
where the last equality follows the properties of Itô Integral and rules of stochastic calculus such that $(\dif \boldsymbol{w})^2 = \dif t$, $\E[ \int_{t=0}^T \dif \boldsymbol{w}_t ] = 0$. Hence, to provide an uncertainty estimate for $\boldsymbol{x}_0$, it remains to estimate the term $V_1$. Recall that the true score function $\nabla_{\boldsymbol{x}_t} \log p(\boldsymbol{x}_t | y)$ is approximated by $ s_\theta((\boldsymbol{x}_t, y, t) = -\boldsymbol{\epsilon}_\theta(\boldsymbol{x}_t, t, y) / \sigma_t$. For ease of notation, let $s_{\theta, t} = s_\theta(\boldsymbol{x}_t, y, t)$ and $\widetilde{s}_{\theta, t} = \widetilde{s}_\theta(\boldsymbol{x}_t, y, t)$, which gives 
\begin{align*}
    V_1
    &= \int_{t=0}^T \int_{s=0}^T \left( \frac{1}{4} \mathrm{Cov}(\boldsymbol{x}_s, \boldsymbol{x}_t) - \frac{1}{2} \mathrm{Cov}(\boldsymbol{x}_s,  s_{\theta, t}) - \frac{1}{2} \mathrm{Cov}(\boldsymbol{x}_t,  s_{\theta, s}) +  \mathrm{Cov}(s_{\theta, t},  s_{\theta, s}) \right)\dif s \dif t.
\end{align*}
When $s \neq t$, score functions $s_{\theta, t}$ and $s_{\theta, s}$ are independent, and similarly, $\boldsymbol{x}_t$ and $s_{\theta, s}$ are also independent. As a result, the above equation can be further simplified as
\begin{align*}
    V_1
    &= \int\limits_{t=0}^T \int\limits_{s=0}^T \left( \frac{1}{4} \mathrm{Cov}(\boldsymbol{x}_s, \boldsymbol{x}_t)  - \frac{1}{2} \mathrm{Cov}(\boldsymbol{x}_s,  s_{\theta, t}) \right)\dif s \dif t -  \int\limits_{t=0}^T \left( \mathrm{Cov}(\boldsymbol{x}_t,  s_{\theta, t}) + \mathrm{Cov}(s_{\theta, t},  s_{\theta, t}) \right) \dif t.
\end{align*}
Combining all the above results together completes the proof of the second statement of \Cref{thm:sde_uncertainty}.

\section{Analysis of Sub-optimality for Black-box Function}
\label{sec:app_sub_optimality}
In this section, we study the behavior of the sub-optimality gap of our algorithm by proving \Cref{thm:y_bound} and \Cref{thm:sub_optimality_gap}. We first introduce the notation that is used throughout this section and the next section. Then we present the main lemmas along with their proofs. Finally, we combine the lemmas to prove our main results.

At each iteration $k \in [K]$, let $y_k^*$ be the target function value on which the diffusion model conditions, 
and $p_\theta$ be the model learned by the conditional diffusion model. We define the performance metric for online BBO problem, which measures the sub-optimal performance gap between the function value achieved by sample $\boldsymbol{x} \sim p_{\theta}(\cdot|y_k^*, \mathcal{D})$ and the target function value $y_k^*$. Its formal definition is described as follows:
\begin{equation}
\label{eqn:sub_optimality_gap}
    \Delta(p_\theta, y_k^*) = \left| y_k^*- \max_{j \in [N]} f(\boldsymbol{x}_j) \right|, ~~ \text{where} ~~\boldsymbol{x}_j \sim p_{\theta}(\cdot|y_k^*, \mathcal{D}),~~\forall j \in [N].
\end{equation}



For simplicity of analysis, we consider $N=1$, and let the generated sample at the $k$-th iteration be $\boldsymbol{x}_k$ in the remaining text. We remark that all proofs go through smoothly for general $N$ with more 
nuanced notations, and do not affect the conclusions being drawn. To proceed with the proofs in this section, we first state the formal assumptions for the black-box function $f(\cdot)$ and sample $\boldsymbol{x}$.
\begin{assumption}
\label{ass:l_lipschitz}
    The scalar black-box function $f$ is L-Lipschitz in $\boldsymbol{x}$:
    \[
        \left| f(\boldsymbol{x}^\prime) - f(\boldsymbol{x})\right| \leq L \|\boldsymbol{x}^\prime - \boldsymbol{x}\|, ~~\forall \boldsymbol{x}^\prime, \boldsymbol{x} \in \R^d.
    \]
\end{assumption}

\begin{assumption}
\label{ass:subGaussian}
    Each generated sample $\boldsymbol{x} \in \R^d$ is $\sigma$-subGaussian. That is, there exists $\sigma \in \R$ such that for any $\boldsymbol{v} \in \R^{d}$ with $\lrn{\boldsymbol{v}} = 1$, $\boldsymbol{v}^{\rT} (\boldsymbol{x} - \E[\boldsymbol{x}])$ is $\sigma$-subGaussian, and its moment generating function is bounded by: 
    \[
       \E[\exp{\left(\lambda \boldsymbol{v}^{\rT} (\boldsymbol{x} - \E[\boldsymbol{x}])\right)}] \leq  \exp{\left( \frac{\sigma^2 \lambda^2}{2} \right)},~~~~\forall \lambda\in\R,~ \boldsymbol{v} \in \mathbb{S}^{d-1},
    \]
    where $\mathbb{S} := \{\boldsymbol{v} \in \R^d: \lrn{\boldsymbol{v}} = 1\}$ is the $(d-1)$ unit sphere.
\end{assumption}

Before proceeding with the proofs of main theorems, we present our main lemmas.

\begin{lemma}
\label{lem:bound_exp_l2norm}
    At each iteration $k \in [K]$, under fixed parameters $\theta$ and $\theta^*$, for $\boldsymbol{x}_k \sim p_\theta(\cdot|y_k^*, \mathcal{D})$, $\boldsymbol{x}^* \sim p_{\theta^*}(\cdot | y_k^*, \mathcal{D})$, we have
    \begin{align}
        \E_{\boldsymbol{x}_k \sim p_\theta(\cdot|y_k^*, \mathcal{D}), \boldsymbol{x}^* \sim p_{\theta^*}(\cdot | y_k^*, \mathcal{D})}\left[ \|\boldsymbol{x}^* - \boldsymbol{x}_k\| \right]
        &\leq 8 \sqrt{d} \sigma + \lrn{ \E_{\boldsymbol{x}^*}[\boldsymbol{x}^*] - \E_{\boldsymbol{x}_k}[\boldsymbol{x}_k] }, \\
        \E_{\boldsymbol{x}_k \sim p_\theta(\cdot|y_k^*, \mathcal{D}), \boldsymbol{x}^* \sim p_{\theta^*}(\cdot | y_k^*, \mathcal{D})}\left[ \|\boldsymbol{x}^* - \boldsymbol{x}_k\| \right]
        & \geq \lrn{ \E_{\boldsymbol{x}^*}[\boldsymbol{x}^*] - \E_{\boldsymbol{x}_k}[\boldsymbol{x}_k] }.
    \end{align}
\end{lemma}
\begin{proof} [Proof of \Cref{lem:bound_exp_l2norm}]
    To bound $\E\left[ \|\boldsymbol{x}^* - \boldsymbol{x}_k\| \right]$, by triangle inequality,
    \begin{align*}
         \E_{\boldsymbol{x}_k, \boldsymbol{x}^*}\left[ \|\boldsymbol{x}^* - \boldsymbol{x}_k\| \right] 
         &= \E\left[ \|\boldsymbol{x}^* - \E[\boldsymbol{x}^*] + \E[\boldsymbol{x}_k] -  \boldsymbol{x}_k + \E[\boldsymbol{x}^*] - \E[\boldsymbol{x}_k]\| \right] \\
         &\leq \E\left[ \lrn{\boldsymbol{x}^* - \E[\boldsymbol{x}^*]} \right] + \E\left[ \lrn{\boldsymbol{x}_k - \E[\boldsymbol{x}_k]} \right] + \E\left[ \lrn{ \E[\boldsymbol{x}^*] - \E[\boldsymbol{x}_k] } \right].
    \end{align*}
    Under \cref{ass:subGaussian}, by \Cref{lem:subGau_exp_norm}, we have,
    \begin{align*}
        \E_{\boldsymbol{x}_k, \boldsymbol{x}^*}\left[ \|\boldsymbol{x}^* - \boldsymbol{x}_k\| \right]
        \leq 8 \sqrt{d} \sigma + \lrn{ \E[\boldsymbol{x}^*] - \E[\boldsymbol{x}_k] }.
    \end{align*}
    Applying triangle inequality completes the step.
    In addition, it can be easily seen that
    \[
        \E_{\boldsymbol{x}_k, \boldsymbol{x}^*}\left[ \|\boldsymbol{x}^* - \boldsymbol{x}_k\| \right]
        \geq \lrn{\E[\boldsymbol{x}^*] - \E[\boldsymbol{x}_k]}.
    \]
\end{proof}

\begin{lemma}
\label{lem:bound_var_l2norm}
    At each iteration $k \in [K]$, under fixed parameters $\theta$ and $\theta^*$, for $\boldsymbol{x}_k \sim p_\theta(\cdot|y_k^*, \mathcal{D})$, $\boldsymbol{x}^* \sim p_{\theta^*}(\cdot | y_k^*, \mathcal{D})$, we have
    \begin{equation}
        \mathrm{Var}_{\boldsymbol{x}_k \sim p_\theta(\cdot|y_k^*, \mathcal{D}), \boldsymbol{x}^* \sim p_{\theta^*}(\cdot | y_k^*, \mathcal{D})}(\lrn{\boldsymbol{x}^* - \boldsymbol{x}_k}) \leq c_3 d \sigma^2.
    \end{equation}
\end{lemma}

\begin{proof} [Proof of \Cref{lem:bound_var_l2norm}]
    By definition of variance, 
    \begin{equation}
    \label{eqn:var_norm_1}
        \mathrm{Var}_{\boldsymbol{x}_k, \boldsymbol{x}^*}(\lrn{\boldsymbol{x}^* - \boldsymbol{x}_k})
        = \E[\lrn{\boldsymbol{x}^* - \boldsymbol{x}_k}^2] - (\E[\lrn{\boldsymbol{x}^* - \boldsymbol{x}_k}])^2.
    \end{equation}
    Expanding the first term leads to
    \begin{align*}
        \E[\lrn{\boldsymbol{x}^* - \boldsymbol{x}_k}^2]
        &= \E[(\boldsymbol{x}^* - \boldsymbol{x}_k)^\rT(\boldsymbol{x}^* - \boldsymbol{x}_k)] \\
        &= \E[\lrn{\boldsymbol{x}^*}^2] + \E[\lrn{\boldsymbol{x}_k}^2] - 2 \E[(\boldsymbol{x}_k)^\rT \boldsymbol{x}^* ] \\
        &= \E[\lrn{\boldsymbol{x}^*}^2] + \E[\lrn{\boldsymbol{x}_k}^2] - 2 \E[(\boldsymbol{x}_k)]^\rT \E[\boldsymbol{x}^*] \numberthis\label{eqn:var_norm_2},
    \end{align*}
    where the last equality is due to the independece between $\boldsymbol{x}^*$ and $\boldsymbol{x}_k$. 
    
    Under \Cref{ass:subGaussian} and by \Cref{lem:subGau_var}, we have
    \begin{align*}
        \E[\lrn{\boldsymbol{x}^*}^2]
        &= \E[\lrn{\boldsymbol{x}^* - \E[\boldsymbol{x}^*] + \E[\boldsymbol{x}^*]}^2] \\
        &= \E[(\boldsymbol{x}^* - \E[\boldsymbol{x}^*])^\rT (\boldsymbol{x}^* - \E[\boldsymbol{x}^*])] + \lrn{\E[\boldsymbol{x}^*]}^2 \\
        &= \mathrm{tr}(\E[(\boldsymbol{x}^* - \E[\boldsymbol{x}^*]) (\boldsymbol{x}^* - \E[\boldsymbol{x}^*])]^\rT) + \lrn{\E[\boldsymbol{x}^*]}^2 \\
        &\leq C d \sigma^2 + \lrn{\E[\boldsymbol{x}^*]}^2.
    \end{align*}
    Here, the second equality holds as the cross terms vanish due to the fact that $\E[\boldsymbol{x}^* - \E[\boldsymbol{x}^*]] = 0$.
    Similarly,
    \[
        \E[\lrn{\boldsymbol{x}_k}^2] \leq C d \sigma^2 + \lrn{\E[\boldsymbol{x}_k]}^2.
    \]
    Substituting the above two results back to \Cref{eqn:var_norm_2},
    \begin{align*}
        \E[\lrn{\boldsymbol{x}^* - \boldsymbol{x}_k}^2]
        &\leq 2C d \sigma^2 + \lrn{\E[\boldsymbol{x}_k]}^2 + \lrn{\E[\boldsymbol{x}^*]}^2 - 2 \E[(\boldsymbol{x}_k)^\rT \boldsymbol{x}^* ] \\
        &\leq 2C d \sigma^2 + \lrn{\E[\boldsymbol{x}_k] - \E[\boldsymbol{x}^*]}^2 \numberthis\label{eqn:var_norm_3}.
    \end{align*}
    Substituting \Cref{eqn:var_norm_3} back to \Cref{eqn:var_norm_1} and applying \Cref{lem:bound_exp_l2norm} leads to
    \begin{align*}
        \mathrm{Var}_{\boldsymbol{x}_k, \boldsymbol{x}^*}(\lrn{\boldsymbol{x}^* - \boldsymbol{x}_k})
        &\leq 2C d \sigma^2 + \lrn{\E[\boldsymbol{x}_k] - \E[\boldsymbol{x}^*]}^2 - (8 \sqrt{d} \sigma + \lrn{ \E[\boldsymbol{x}^*] - \E[\boldsymbol{x}_k]})^2 \leq  c_3 d \sigma^2.
    \end{align*}
\end{proof}

With the above results, we are ready to prove \Cref{thm:y_bound} and \Cref{thm:sub_optimality_gap}.

\yBoundExp*
\begin{proof} [Proof of \Cref{thm:y_bound}]
    Recall that we consider the case where $N=1$, and denote $\boldsymbol{x}_k$ the generated sample in the $k$-th iteration, i.e. $\boldsymbol{x}_k \sim p_\theta(\cdot|y_k^*, \mathcal{D})$, where $\theta \sim p(\theta \mid \mathcal{D})$. In each iteration $k$, with the existence of $\theta^* \sim p(\theta \mid \mathcal{D})$, we have $y_k^* = f(\boldsymbol{x}^*)$, where $\boldsymbol{x}^* \sim p_{\theta^*}(\cdot | y_k^*, \mathcal{D})$. Hence, under \Cref{ass:l_lipschitz},
    \begin{align*}
         \E\left[{\Delta(p_\theta, y_k^*)}\right]
         = \E\left[\left|f(\boldsymbol{x}^*) - f(\boldsymbol{x}_k )\right|\right] \leq L \E\left[ \|\boldsymbol{x}^* - \boldsymbol{x}_k\| \right].
    \end{align*}

    By \Cref{lem:bound_exp_l2norm}, tower rule and \Cref{lem:posterior}, we have
    \begin{align*}
        \E\left[{\Delta(p_\theta, y_k^*)}\right] 
        &\leq L \E_{\theta, \theta^*}\left[ \E_{x, x^*} \left[\|\boldsymbol{x}^* - \boldsymbol{x}_k\| \right] | \theta, \theta^* \right] \\
        &\leq 8 L \sqrt{d} \sigma + \E_{\theta, \theta^*} \left[ \lrn{ \E_{\boldsymbol{x}^*}[\boldsymbol{x}^* | \theta^* ] - \E_{\boldsymbol{x}_k}[\boldsymbol{x}_k | \theta ] } \right] \\
        &\leq c_1 L \sqrt{d} \sigma.
    \end{align*}
\end{proof}

\yBoundVar*

\begin{proof} [Proof of \Cref{thm:sub_optimality_gap}]
    At every iteration $k \in [K]$, let the target function value on which the conditional diffusion model conditions be $y_k^*$. The statement needs to hold for each conditional diffusion model in the ensemble, and thus for simplicity of notation, the subscript $i$ of $\theta_i$ is dropped in the remaining proof. With the existence of $\theta^* \sim p(\theta \mid \mathcal{D})$, we have $y_k^* = f(\boldsymbol{x}^*)$, where $\boldsymbol{x}^* \sim p_{\theta^*}(\cdot | y_k^*, \mathcal{D})$. Recall that $f(\boldsymbol{x}_k)$ is achieved by $\boldsymbol{x}_k \sim p_\theta(\cdot|y_k^*, \mathcal{D})$, where $\theta \sim p(\theta \mid \mathcal{D})$, and $N=1$. 
    
    Thus, by Eve's law, the overall variance of $\Delta(p_\theta, y_k^*)$ can be decomposed as:
    \begin{align*}
        \mathrm{Var}\left(\Delta(p_\theta, y_k^*)\right)
        &= \mathrm{Var}\left(|y_k^* - f(\boldsymbol{x}_k)|\right) \\
        &= \mathrm{Var}\left(|f(\boldsymbol{x}^*) - f(\boldsymbol{x}_k)|\right) \\
        &= \underbrace{\E_{\theta, \theta^*}\left[ \mathrm{Var}_{\boldsymbol{x}_k, \boldsymbol{x}^*}(|f(\boldsymbol{x}^*) - f(\boldsymbol{x}_k)|~\rvert~\theta, \theta^*) \right]}_{T_1} + \underbrace{\mathrm{Var}_{\theta, \theta^*}(\E_{\boldsymbol{x}_k, \boldsymbol{x}^*}[|f(\boldsymbol{x}^*) - f(\boldsymbol{x}_k)|~\rvert~\theta, \theta^*])}_{T_2}.
    \end{align*}
   In particular, the first term $T_1$ corresponds to the aleatoric component and the second term $T_2$ corresponds to the episdemic component. We then proceed to bound the above two terms separately. 

    \textbf{Step 1: bound $T_1$.}
    Under \Cref{ass:l_lipschitz},
    \[
        \mathrm{Var}_{\boldsymbol{x}_k, \boldsymbol{x}^*}(|f(\boldsymbol{x}^*) - f(\boldsymbol{x}_k)|~\rvert~\theta, \theta^*) \leq L^2 \mathrm{Var}_{\boldsymbol{x}_k, \boldsymbol{x}^*}(\lrn{\boldsymbol{x}^* - \boldsymbol{x}_k}|~\rvert~\theta, \theta^*).
    \]
    Under \Cref{ass:subGaussian} and by \Cref{lem:bound_var_l2norm},
    \begin{equation}
    \label{eqn:bound_var_1}
        T_1 \leq L^2 \E_{\theta, \theta^*}[\mathrm{Var}_{\boldsymbol{x}_k, \boldsymbol{x}^*}(\lrn{\boldsymbol{x}^* - \boldsymbol{x}_k}|~\rvert~\theta, \theta^*)] \leq c_3 L^2 d\sigma^2.
    \end{equation}
    \textbf{Step 2: bound $T_2$.}
    Under \Cref{ass:l_lipschitz},
   \[
        T_2
        \leq L^2 \mathrm{Var}_{\theta, \theta^*}(\E_{\boldsymbol{x}_k, \boldsymbol{x}^*}[\lrn{\boldsymbol{x}^* - \boldsymbol{x}_k}|~\rvert~\theta, \theta^*]))
   \]
   
    By \Cref{lem:bound_exp_l2norm}, 
    \begin{align*}
        \mathrm{Var}_{\theta, \theta^*}(\E_{\boldsymbol{x}_k, \boldsymbol{x}^*}[|f(\boldsymbol{x}^*) - f(\boldsymbol{x}_k)|~\rvert~\theta, \theta^*])
        &\leq \mathrm{Var}_{\theta, \theta^*}\left(\E_{\theta, \theta^*}\left[8 \sqrt{d} \sigma + \lrn{ \E_{\boldsymbol{x}^*}[\boldsymbol{x}^*] - \E_{\boldsymbol{x}_k}[\boldsymbol{x}_k] }\right]\right) \\
        & \leq \mathrm{Var}_{\theta, \theta^*}\left(\lrn{ \E_{\boldsymbol{x}^*}[\boldsymbol{x}^*] - \E_{\boldsymbol{x}_k}[\boldsymbol{x}_k] }\right). 
    \end{align*}

    Then by property of variance, we have
    \begin{align*}
        \mathrm{Var}_{\theta, \theta^*}\left(\lrn{ \E_{\boldsymbol{x}^*}[\boldsymbol{x}^*] - \E_{\boldsymbol{x}_k}[\boldsymbol{x}_k] }\right)
        &= \E_{\theta, \theta^*}\left[ \lrn{ \E_{\boldsymbol{x}^*}[\boldsymbol{x}^*] - \E_{\boldsymbol{x}_k}[\boldsymbol{x}_k] }^2 \right] - \Big(\E_{\theta, \theta^*} \Big[ \lrn{ \E_{\boldsymbol{x}^*}[\boldsymbol{x}^*] - \E_{\boldsymbol{x}_k}[\boldsymbol{x}_k] } \Big]\Big)^2.
    \end{align*}
    From the proof of \Cref{lem:bound_var_l2norm}, we have
    \begin{align*}
        &\quad\E_{\theta, \theta^*}\Big[\lrn{ \E_{\boldsymbol{x}^*}[\boldsymbol{x}^* | \theta^*] - \E_{\boldsymbol{x}_k}[\boldsymbol{x}_k | \theta] }^2 \Big] \\
        &= \E_{\theta^*}[\E_{\boldsymbol{x}^*}[\lrn{\boldsymbol{x}^*}^2 | \theta^*]] + \E_{\theta}[\E_{\boldsymbol{x}_k}[\lrn{\boldsymbol{x}_k}^2 |\theta]] - 2 \E_{\theta, \theta^*}[\E_{\boldsymbol{x}_k}[(\boldsymbol{x}_k|\theta)]^\rT \E_{\boldsymbol{x}^*}[\boldsymbol{x}^*|\theta^*]] \\
        &= 2(\E_{\theta}[\E_{\boldsymbol{x}_k}[\lrn{\boldsymbol{x}_k}^2 |\theta]] - \E_{\theta, \theta^*}[\E_{\boldsymbol{x}_k}[(\boldsymbol{x}_k|\theta)]^\rT \E_{\boldsymbol{x}^*}[\boldsymbol{x}^*|\theta^*]]) \\
        &= 2(\E_{\theta}[\E_{\boldsymbol{x}_k}[\lrn{\boldsymbol{x}_k}^2 |\theta]] - \lrn{\E_{\theta}[\E_{\boldsymbol{x}_k}[\boldsymbol{x}_k | \theta]]}^2) \\
        & = 2 \mathrm{Var}_{\theta}(\E_{\boldsymbol{x}_k}[\lrn{\boldsymbol{x}_k]}),
    \end{align*}
    where the third equality is by the law of total expectation and the fact that $\E_{\theta}[\E_{\boldsymbol{x}_k}[\boldsymbol{x}_k | \theta]] = \E_{\theta^*}[\E_{\boldsymbol{x}^*}[\boldsymbol{x}^* | \theta^*]]$. Combining the above results, we have
    \begin{equation}
    \label{eqn:bound_var_2}
        T_2 \leq L^2 \mathrm{Var}_{\theta, \theta^*}\left(\lrn{ \E_{\boldsymbol{x}^*}[\boldsymbol{x}^*] - \E_{\boldsymbol{x}_k}[\boldsymbol{x}_k] }\right)
        \leq 2 L^2 \mathrm{Var}_{\theta}(\E_{\boldsymbol{x}_k}[\lrn{\boldsymbol{x}_k]}).
    \end{equation}
    Combining \Cref{eqn:bound_var_1} and \Cref{eqn:bound_var_2} completes the proof:
    \begin{align*}
        \mathrm{Var}\left(\Delta(p_\theta, y_k^*)\right) 
        \leq c_3 L^2 d \sigma^2 + 2 L^2 \text{Var}_\theta(\E_{\boldsymbol{x}}[\lrn{\boldsymbol{x}_k}]).
    \end{align*} 
\end{proof}



\subsection{Supporting Lemmas}
\begin{lemma} [\cite{wainwright2019high}]
\label{lem:subGau_exp_norm}
    Let $\boldsymbol{x} \in \R^d$ be a $\sigma$-subGaussian ramdom vector, then
    \begin{equation}
        \E[\lrn{\boldsymbol{x} - \E[\boldsymbol{x}]}] \leq 4 \sigma \sqrt{d}.
    \end{equation}
\end{lemma}

\begin{lemma}
\label{lem:subGau_var}
    Let $\boldsymbol{x} \in \R^d$ be a $\sigma$-subGaussian ramdom vector, then its variance satisfies:
    \begin{equation}
        \text{Var}[\boldsymbol{x}] \leq C d \sigma^2,
    \end{equation}
    where C is some positive constant.
\end{lemma}
\begin{proof} [Proof of \cref{lem:subGau_var}]
    By definition of sub-Gaussian vector, for any direction $\boldsymbol{u} \in \R^d$ with $\lrn{\boldsymbol{u}} = 1$,
    \begin{equation*}
        \E\left[\exp(\lambda \boldsymbol{u}^\rT(\boldsymbol{x} - \E[\boldsymbol{x}]))\right] \leq \exp \left( \frac{\lambda^2\sigma^2}{2} \right), ~~\forall \lambda \in \R.
    \end{equation*}
    This implies that the second moment in any direction
    $\boldsymbol{u}$ satisfies:
    \begin{equation*}
        \E\left[\boldsymbol{u}^\rT((\boldsymbol{x} - \E[\boldsymbol{x}])(\boldsymbol{x} - \E[\boldsymbol{x}])^\rT)\right] \leq \sigma^2.
    \end{equation*}
    Therefore, the maximum eigenvalue of the covariance matrix is upper-bounded by $C \sigma^2$, where $C$ is some positive constant.
    \begin{align*}
        \text{Var}[\boldsymbol{x}] 
        = \text{tr}\left(\E\left[ (\boldsymbol{x} - \E[\boldsymbol{x}])(\boldsymbol{x} - \E[\boldsymbol{x}])^\rT \right]\right) \leq C d\sigma^2.
    \end{align*}
\end{proof}

\begin{lemma}
\label{lem:posterior}
        In each iteration $k \in [K]$, let $\mathcal{D}$ be the collected dataset, $\theta$ and $\theta^*$ are parameters independently drawn from posterior $p(\theta | \mathcal{D})$, $\boldsymbol{x}_k \sim p_\theta(\cdot|y_k^*, \mathcal{D})$ and $\boldsymbol{x}^* \sim p_{\theta^*}(\cdot|y_k^*, \mathcal{D})$. For any measurable function $f$, and $\sigma(\mathcal{D})$-measurable random variable $\boldsymbol{x_k}$, 
    \begin{align*}
        \E\left[f(\boldsymbol{x}_k) \right] 
        &= \E\left[f(\boldsymbol{x}^*) \right].
    \end{align*}
\end{lemma}
\begin{proof} [Proof of \Cref{lem:posterior}]
    Since the black-box function $f$ is measurable, and by the nature of \Cref{alg:diff_bo}, in each iteration $k$, the generated sample $\boldsymbol{x_k}$, the target function value $y_k^*$, the predictive distribution $p_\theta(\cdot|y_k^*, \mathcal{D})$, the posterior distribution $p(\theta \mid \mathcal{D})$ are $\sigma(\mathcal{D})$-measurable at iteration $k$, the only randomness in $f(\boldsymbol{x})$ comes from the random sampling in the algorithm. Thus, condition on the training data $\mathcal{D}$ and target value $y_k^*$, by tower rule,
    \begin{align*}
        \E\left[ f(\boldsymbol{x}_k) \right]
        = \E\left[\E\left[ f(\boldsymbol{x}_k) | \theta \right] \right] 
        &= \int_\theta \int_{\boldsymbol{x}_k} f(\boldsymbol{x}_k) p_\theta(\boldsymbol{x}_k | y_k^*, \mathcal{D}) p(\theta | \mathcal{D})  \dif \boldsymbol{x}_k \dif \theta \\
        &= \int_\theta \int_{\boldsymbol{x}_k} f(\boldsymbol{x}_k) p_\theta(\boldsymbol{x}_k | y_k^*, \mathcal{D}) \dif \boldsymbol{x}_k ~ p(\theta | \mathcal{D}) \dif \theta. 
    \end{align*}
    Note that both the true parameter $\theta^*$ and the chosen parameter $\theta$ are drawn from the same posterior distribution $p(\theta \mid \mathcal{D})$, we have
    \begin{align*}
       \int_\theta \int_{\boldsymbol{x}} f(\boldsymbol{x}) p_\theta(\boldsymbol{x} | y_k^*, \mathcal{D}) \dif \boldsymbol{x} ~ p(\theta | \mathcal{D}) \dif \theta
       = \int_{\theta^*} \int_{\boldsymbol{x}} f(\boldsymbol{x}) p_{\theta^*}(\boldsymbol{x} | y_k^*, \mathcal{D}) \dif \boldsymbol{x} ~ p({\theta^*} | \mathcal{D}) \dif {\theta^*}.
    \end{align*}
    As a result, we have
    \begin{align*}
        \E\left[ f(\boldsymbol{x}_k) \right] 
        &= \int_{\theta^*} \int_{\boldsymbol{x}^*} f(\boldsymbol{x}^*) p_{\theta^*}(\boldsymbol{x}^* | y_k^*, \mathcal{D}) \dif \boldsymbol{x}^*  p({\theta^*} | \mathcal{D}) \dif {\theta^*} 
        =\E\left[\E\left[ f(\boldsymbol{x}^*) | \theta^* \right] \right] 
        = \E\left[ f(\boldsymbol{x}^*) \right].
    \end{align*}
\end{proof}

\begin{corollary}
\label{cor:posterior_norm}
        In each iteration $k \in [K]$, let $\mathcal{D}$ be the collected dataset, $\theta$ and $\theta^*$ are parameters independently drawn from posterior $p(\theta | \mathcal{D})$, $\boldsymbol{x}_k \sim p_\theta(\cdot|y_k^*, \mathcal{D})$ and $\boldsymbol{x}^* \sim p_{\theta^*}(\cdot|y_k^*, \mathcal{D})$. For any measurable function $f$, and $\sigma(\mathcal{D})$-measurable random variable $\boldsymbol{x_k}$, 
    \begin{align*}
        \E\left[\lrn{\boldsymbol{x}_k} \right] 
        &= \E\left[\lrn{\boldsymbol{x}^*} \right].
    \end{align*}
\end{corollary}
\begin{proof} [Proof of \Cref{cor:posterior_norm}]
    Since the norm function is deterministic and $\sigma(\mathcal{D})$-measurable, the proof directly follows that of \Cref{lem:posterior}.
\end{proof}

\section{Optimality of Proposed Acquisition Function}
\label{sec:app_optimization}
\optiEqn*

\begin{proof} [Proof of \Cref{thm:optimization_equivalence}]
    Following \Cref{thm:sub_optimality_gap}, we can express the function evaluation as follows,
\begin{align*}
     f(\boldsymbol{x}_k)
    = y_k - (y_k - f(\boldsymbol{x}_k)), \forall k \in [K].
\end{align*}

The overall objective of the optimization problem defined in \Cref{eqn:optimization_objective} can then be further decomposed as 
\begin{align*}
    &~~~~~\max_{y_k \in \R} \sum_{k=1}^K f(\boldsymbol{x}_k), ~~ \boldsymbol{x}_k \sim p_{\theta}(\cdot \mid y_k),~~\theta \in \Theta \\
    &\Leftrightarrow \max_{y_k \in \R} \sum_{k=1}^K y_k - (y_k - f(\boldsymbol{x}_k)), ~~ \boldsymbol{x}_k \sim p_{\theta}(\cdot \mid y_k),~~\theta \in \Theta \\
    &\Rightarrow \max_{y_k \in \mathcal{Y}} \sum_{k=1}^K y_k - \Delta(p_\theta, y_k),
\end{align*}
where the candidate set $\mathcal{Y}$ is constructed based on the training data and is designed to explore the objective space efficiently. When considering an online maximization problem, adding a positive term would lead to overestimation, because the model would be overly optimistic about $f(\boldsymbol{x}_k)$. Therefore, we should only consider the case where we the uncertainty is being subtracted. By \Cref{thm:sub_optimality_gap}, which shows $\Delta(p_\theta, y_k^*)$ can be effectively upper bounded the epidemic uncertainty, we therefore have
\begin{align*}
    \max_{y_k \in \R} \sum_{k=1}^K f(\boldsymbol{x}_k), ~~ \boldsymbol{x}_k \sim p_{\theta}(\cdot \mid y_k),~~\theta \in \Theta
    &\Rightarrow \max_{y_k \in \mathcal{Y}} \sum_{k=1}^K y_k - \Delta_{\mathrm{episdemic}}(y_k, \mathcal{D}).
\end{align*}
    Essentially, our chosen acquisition function allows \ours{} to maximize the lower bound of the original optimization problem. Penalizing high uncertainty ensures that the model prioritizes more confident predictions (i.e. those with lower epistemic uncertainty), which are more likely to yield higher objective function values.
\end{proof}

\section{Experiment Details}
\label{sec:app_experiment}
\subsection{Dataset Details.} \label{sec:app_dataset}\textbf{DesignBench} \citep{trabucco2022design} is a benchmark for real-world black-box optimization tasks. For continuouse tasks, we use Superconductor, D'Kitty Morphology and Ant Morphology benchmarks. For discrete tasks, we utilize TFBind8 and TFBind10 benchmarks. We exclude Hopper due to the domain is known to be buggy, as explained in Appendix C in \citep{krishnamoorthy2023diffusion}. We also exclude NAS due to the significant computational resource requirement. Additionally, we exclude the ChEMBL task because the oracle model exhibits non-trivial discrepancies when queried with the same design.
\begin{itemize}
    \item \textbf{Superconductor (materials optimization).} This task involves searching for materials with high critical temperatures. The dataset comprises 17,014 vectors, each with 86 components that represent the number of atoms of each chemical element in the formula. The provided oracle function is a pre-trained random forest regression model.
    
    \item \textbf{D’Kitty Morphology (robot morphology optimization).} This task focuses on optimizing the parameters of a D’Kitty robot, including the size, orientation, and location of the limbs, to make it suitable for a specific navigation task. The dataset consists of 10,004 entries with a parameter dimension of 56. It utilizes MuJoCO \citep{todorov2012mujoco}, a robot simulator, as the oracle function.
    
    \item \textbf{Ant Morphology (robot morphology optimization).} Similar to D’Kitty, this task aims to optimize the parameters of a quadruped robot to maximize its speed. It includes 10,004 data points with a parameter dimension of 60. It also uses MuJoCO as the oracle function.
    
    \item \textbf{TFBind8 (DNA sequence optimization).} This task seeks to identify the DNA sequence of length eight with the highest binding affinity to the transcription factor SIX6 REF R1. The design space comprises sequences of nucleotides represented as categorical variables. The dataset size is 32,898, with a dimension of 8. The ground truth is used as a direct oracle since the affinity for the entire design space is available.
    
    \item \textbf{TFBind10 (DNA sequence optimization).} Similar to TFBind8, this task aims to find the DNA sequence of length ten that exhibits the highest binding affinity with transcription factor SIX6 REF R1. The design space consists of all possible nucleotide sequences. The dataset size is 10,000, with a dimension of 10. The ground truth is used as a direct oracle since the affinity for the entire design space is available.
\end{itemize}

\paragraph{Molecular Discovery.} A key problem in drug discovery is the optimization of a compound's activity against a biological target with therapeutic value. Similar to other papers \citep{eckmann2022limo, jeon2020autonomous, lee2023exploring, noh2022path}, we attempt to optimize the score from AutoDock4 \citep{morris2009autodock4}, which is a physics-based estimator of binding affinity. The oracle is a feed-forward model as a surrogate to AutoDock4. The surrogate model is trained until convergence on 10,000 compounds randomly sampled from the latent space (using $\mathcal{N}(0, 1)$) and their computed objective values with AutoDock4. We construct our continuous design space by fixing a random protein embedding and randomly sampling 10,000 molecular embedding of dimension $32$. 

For all the tasks, we sort the offline dataset based on the objective values and select data from the $25\%$ to $50\%$ as the initial training dataset. We use data with lower objective scores to better observe performance differences across each baseline.
The overview of all the task statistics is provided in \Cref{tab:task_data}.

\begin{table}[h!]
    \centering
    \begin{tabular}{lrrr}
        \toprule
        \textbf{Task} & \textbf{Size} & \textbf{Dimensions} & \textbf{Task Max} \\
        \midrule
        TFBind8 & 32,898 & 8 & 1.0 \\
        TFBind10 & 10,000 & 10 & 2.128 \\
        D’Kitty & 10,004 & 56 & 340.0 \\
        Ant & 10,004 & 60 & 590.0 \\
        Superconductor & 17,014 & 86 & 185.0 \\
        Molecular Discovery & 10,000 & 32 & 1.0 \\
        \bottomrule
    \end{tabular}
        \caption{Data Statistics}
    \label{tab:task_data}
\end{table}

\begin{figure}[h]
    \centering
    \includegraphics[width=0.5\linewidth]{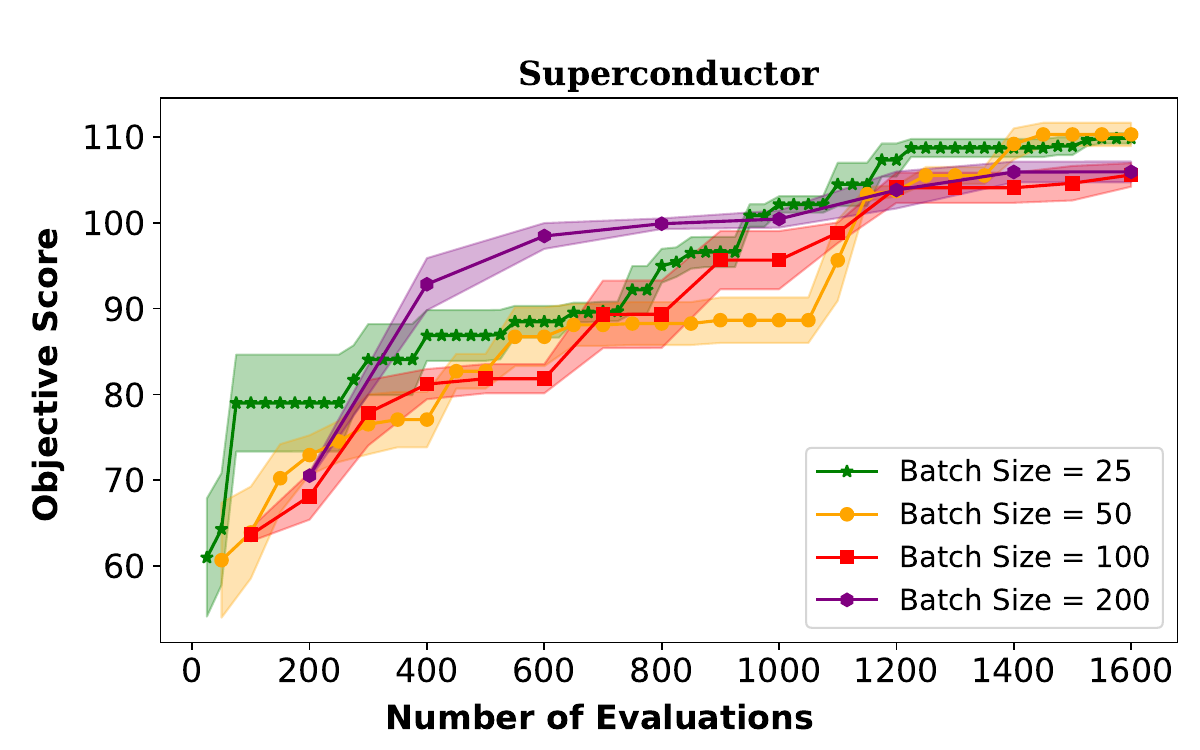}
    \caption{Ablation study to evaluate the effect of batch size on the superconductor task. The mean and standard deviation across three random seeds are plotted. \ours{} shows robust performances across different batch size given the same total number of evaluations.}
    \label{fig:ablation2}
\end{figure}

\subsection{Implementation Details.} 
We train our model on one NVIDIA A100 GPU and report the average performance over 3 random runs along with standard deviation for each task. For discrete tasks, we follow the procedure in \cite{krishnamoorthy2023diffusion} where we convert the $d$-dimensional vector to a $d \times c$ one hot vector regarding $c$ classes. We then approximate logits by interpolating between a uniform distribution and the one hot distribution using a mixing factor of $0.6$. We jointly train a conditional and unconditional model with the same model by randomly set the conditioning value to 0 with dropout probability of $0.15$.

For each task, we fix the learning rate at $0.001$ with batch size of $256$. We use $5$ ensemble models to estimate the uncertainty for our acquisition function. We set hidden dimensions to $1024$ and gamma to $2$. We use $10\%$ of the available data at each iteration as validation set during training.

\section{Ablation Study}
\label{section:ablation_appendices}

We evaluate the effect of batch size, i.e., the number of queries per iteration on \ours{} on the Superconductor task. As shown in \Cref{fig:ablation2}, we compare the objective function score over number of function evaluations. We can see the performance of our approach remains similar when the batch size becomes larger, suggesting remarkable robustness across different batch sizes. Hence, \ours{} is a highly-scalable inverse modeling approach that can efficiently leverage parallelism to handle larger computational loads without compromising performance.

\end{document}